\documentclass[twoside]{article}

\usepackage[accepted]{aistats2022}
\usepackage{mymath}

\usepackage[round]{natbib}

\begin{document}

%
\runningtitle{Gap-Dependent Unsupervised Exploration for Reinforcement Learning}

%

\twocolumn[

\aistatstitle{Gap-Dependent Unsupervised Exploration \\ for Reinforcement Learning}

\aistatsauthor{ Jingfeng Wu \And Vladimir Braverman \And  Lin F. Yang }

\aistatsaddress{ Johns Hopkins Universit \\ \texttt{uuujf@jhu.edu} \And  Johns Hopkins University \\ \texttt{vova@cs.jhu.edu} \And University of California, Los Angeles \\ \texttt{linyang@ee.ucla.edu}} 

]

\begin{abstract}
For the problem of task-agnostic reinforcement learning (RL), an agent first collects samples from an unknown environment without the supervision of reward signals, then is revealed with a reward and is asked to compute a corresponding near-optimal policy. Existing approaches mainly concern the worst-case scenarios, in which no structural information of the reward/transition-dynamics is utilized. Therefore the best sample upper bound is $\propto\widetilde{\mathcal{O}}(1/\epsilon^2)$, where $\epsilon>0$ is the target accuracy of the obtained policy, and can be overly pessimistic. To tackle this issue, we provide an efficient algorithm that utilizes a gap parameter, $\rho>0$, to reduce the amount of exploration. In particular, for an unknown finite-horizon Markov decision process, the algorithm takes only $\widetilde{\mathcal{O}} (1/\epsilon \cdot (H^3SA / \rho + H^4 S^2 A) )$ episodes of exploration, and is able to obtain an $\epsilon$-optimal policy for a post-revealed reward with sub-optimality gap at least $\rho$, where $S$ is the number of states, $A$ is the number of actions, and $H$ is the length of the horizon, obtaining a nearly \emph{quadratic saving} in terms of $\epsilon$. We show that,  information-theoretically, this bound is nearly tight for $\rho < \Theta(1/(HS))$ and $H>1$. We further show that $\propto\widetilde{\mathcal{O}}(1)$ sample bound is possible for $H=1$ (i.e., multi-armed bandit) or with a sampling simulator, establishing a stark separation between those settings and the RL setting. 
\end{abstract}

\allowdisplaybreaks

\section{INTRODUCTION}\label{sec:introduction}
\emph{Unsupervised exploration} is an emergent and challenging topic for reinforcement learning (RL) that inspires research interests in both application~\citep{riedmiller2018learning,finn2017deep,xie2018few,xie2019improvisation,schaul2015universal,riedmiller2018learning} and theory \citep{hazan2018provably,jin2020reward,zhang2020task,kaufmann2020adaptive,menard2020fast,zhang2020nearly,wu2020accommodating,wang2020reward}.
The formal formulation of an unsupervised RL problem consists of an \emph{exploration phase} and a \emph{planning phase}~\citep{jin2020reward}: in the exploration phase, an agent interacts with the unknown environment without the supervision of reward signals; then in the planning phase, the agent is prohibited to interact with the environment, and is required to compute a nearly optimal policy for some revealed reward function based on its exploration experiences. In particular, if the reward function is \emph{independent} of the agent's exploration policy, the problem is called \emph{task-agnostic exploration} (TAE)~\citep{zhang2020task}; and if the reward function is chosen \emph{adversarially} according to the agent's exploration policy, the problem is called \emph{reward-free exploration} (RFE)~\citep{jin2020reward}.
The performance of an unsupervised exploration algorithm is measured by the \emph{sample complexity}, i.e., the number of samples the algorithm needs to collect during the exploration phase in order to complete the planning task near-optimally up a small  error (with high probability, see Section \ref{sec:setups} for a formal definition).
Existing algorithms for unsupervised RL exploration~\citep{jin2020reward,zhang2020task,wu2020accommodating,zhang2020nearly,wang2020reward} suffer a sample complexity (upper bounded by) $\propto\widetilde{\Ocal}(1/\epsilon^2)$\footnote{Here we use $\propto\widetilde{\Ocal}(\cdot)$ to emphasize the rates' dependence on $\epsilon$, where the other parameters are treated as constants. Similarly hereafter.} for a target planning error tolerance $\epsilon$.
In a worst-case consideration, this rate, in terms of dependence on $\epsilon$, is known to be unimprovable except for logarithmic factors \citep{jin2020reward,dann2015sample}.

However, the above worst-case sample bounds can be \emph{overly pessimistic} in practical scenarios, since the planning task is often known to be a benign instance, even though the exploration phase is conducted under no supervision from rewards.
In particular, the reward function revealed in the planning phase could induce a constant \emph{minimum nonzero sub-optimality gap} (or simply \emph{gap}, that is the minimum gap between the best action and the second best action in the optimal $Q$-value function, and is defined formally in  Section~\ref{sec:setups})~\citep{tewari2007optimistic,ortner2007logarithmic,ok2018exploration} such that the planning task is essentially an ``easy'' one~\citep{simchowitz2019non}.
More importantly, a \emph{reward-agnostic} gap parameter, e.g., an uniform lower bound on the gaps of possibly revealed reward functions, could be available to facilitate the exploration process.
See following for an example.

\paragraph{An Example.}
Let us consider training an agent for Go-Game (vs. an unknown player), where the winning rule could be either the Chinese rule, the Japanese rule or the Korean rule.
The agent will be rewarded $1$ for winning and $0$ for losing,  
and its goal is to explore without knowing the winning condition and to provide a solution to one/any of these rules specified afterward.
Note that all these winning conditions have an uniform, constant gap lower bound ($\approx 1$, assuming that the opponent plays nearly deterministically). 
Note that this RL problem is still unsupervised as the winning condition is unknown during exploration; but the uniform gap lower bound could potentially be exploited to accelerate exploration.

\paragraph{Open Problem.}
In the supervised RL setting where the reward signals are available, a constant gap significantly improves the sample complexity bounds, e.g., from $\propto\widetilde{\Ocal}(1/\epsilon^{2})$ to $\propto\widetilde{\Ocal}(1)$ \citep{jaksch2010near,simchowitz2019non,yang2020q,he2020logarithmic,xu2021fine}.
However, the following question remains open for unsupervised RL:

\begin{center}
\emph{Can unsupervised RL problems be solved faster when provided with a reward-agnostic gap parameter?}
\end{center}

\paragraph{A Case Study on Multi-Armed Bandit.}
To gain more intuition, let us take a quick look at the (gap-dependent) unsupervised exploration problem for multi-armed bandit (MAB).
In the worst-case setup, there is a minimax lower bound $\propto \Omega(1/\epsilon^2)$ for unsupervised exploration on an MAB instance \citep{mannor2004sample}, where $\epsilon$ is a small tolerance for the planning error. 
On the other hand, if the MAB instance has a constant gap, a rather simple \emph{uniform exploration} strategy achieves $\propto\widetilde{\Ocal}(1)$ sample complexity upper bound (see, e.g., Theorem 33.1 in \citep{lattimore2020bandit}, or Appendix \ref{appendix-section:bandit-and-mdp-with-generative-model}).
This example provides positive evidence that a constant gap parameter could accelerate unsupervised RL, too.


\paragraph{Our Contributions.}
In this paper, we study the \emph{gap-dependent task-agnostic exploration} (gap-TAE) problem on a finite-horizon Markov decision process (MDP) with $S$ states, $A$ actions and $H \ge 2$ decision steps per episode. 
We consider a variant of upper-confidence-bound (UCB) algorithm that explores the unknown environment through a greedy policy that minimizes the cumulative exploration bonus \citep{zhang2020task,wang2020reward,wu2020accommodating}; our exploration bonus is of UCB-type, but is \emph{clipped} according to the gap parameter.
Theoretically, we show that $\widetilde{\mathcal{O}} (H^3SA / (\rho \epsilon) + H^4 S^2 A / \epsilon )$ number of trajectories is sufficient for the proposed algorithm to plan $\epsilon$-optimally for a task with a gap parameter $\rho$, where $\epsilon>0$ is the planning error parameter. 
This fast rate $\propto \widetilde{\mathcal{O}}(1/\epsilon)$ improves the existing, pessimistic rates $\propto \widetilde{\mathcal{O}}(1/\epsilon^2)$ \citep{zhang2020task,wang2020reward,wu2020accommodating} significantly when $\epsilon\ll \rho$.
Furthermore, we provide an information-theoretic lower bound, $\Omega (H^2SA / (\rho \epsilon) )$, on the number of trajectories required to solve the problem of gap-TAE on MDPs with $H \ge 2$.
This indicates that, for gap-TAE on MDP with $H \ge 2$, the $\propto \widetilde{\Ocal}(1/\epsilon)$ rate achieved by our algorithm is nearly the best possible.
These results naturally extend to other unsupervised RL settings.

Interestingly, our results imply that RL is \emph{statistically harder} than MAB in the setting of gap-dependent unsupervised exploration. 
In particular, a finite-horizon MDP with $H=1$ reduces to an MAB problem, where it is known that $\propto\widetilde{\Ocal}(1)$ samples are sufficient for solving gap-TAE; however when $H\ge 2$ which corresponds to the general RL setting, our results show that at least $\propto{\Omega}(1/\epsilon)$ amount of samples are required for solving gap-TAE.
This is against an emerging wisdom from the supervised RL theory, that RL ($H \ge 2$) is statistically as easy as learning MAB ($H=1$) when the $H$ factor is normalized, ignoring logarithmic factors~\citep{jiang2018open,wang2020long,zhang2020reinforcement}.

\paragraph{Notations.}
For two functions $f(x) \ge 0$ and $g(x) \ge 0$ defined for $x \in [0, \infty)$,
we write $f(x) \lesssim g(x)$ if $f(x) \le c\cdot g(x)$ for some absolute constant $c > 0$; 
we write $f(x) \gtrsim g(x)$ if $g(x) \lesssim f(x)$;
and we write $f(x) \eqsim g(x)$ if $f(x) \lesssim g(x)\lesssim f(x)$.
Moreover, we write $f(x) = \Ocal (g(x))$ if $\lim_{x\to\infty} f(x) / g(x) < c$ for some absolute constant $c  >0$;
we write $f(x) = \Omega(g(x))$ if $g(x) = \Ocal (f(x))$;
and we write $f(x) = \Theta(g(x))$ if $f(x) = \Ocal (g(x))$ and $g(x) = \Ocal(f(x))$.
To hide the logarithmic factors, we write $f(x) = \widetilde{\Ocal} (g(x)) $ if $f(x) = \Ocal(g(x) \log^d x) $ for some absolute constant $d > 0$.
For $a,b \in \Rbb$, we write $a\land b := \min\{a,b\}$ and $a\lor b := \max\{a,b\}$.
For a positive integer $H$, we define $[H] := \{1,2,\dots, H\}$.


\section{PROBLEM SETUP}\label{sec:setups}

\paragraph{Finite-Horizon MDP.}
We focus on \emph{finite-horizon Markov decision process} (MDP), which is  specified by a tuple, $\bracket{\Scal, \Acal, H, \Pbb, x_1, r}$.
$\Scal$ is a finite \emph{state} set where $\abs{\Scal} = S$.
$\Acal$ is a finite \emph{action} set where $\abs{\Acal} = A$.
$H$ is the \emph{length of the horizon}.
$\Pbb: \Scal\times\Acal \to [0,1]^\Scal$ is an \emph{unknown, stationary transition probability}.
Without lose of generality, we assume the MDP has fixed initial state $x_1$\footnote{We may as well consider MDPs with an external initial state $x_0$ with zero reward for all actions, and a transition $\Pbb_0 (\cdot \mid x_0, a) = \Pbb_0(\cdot) $ for all action $a$, which is equivalent to our setting by letting the horizon length $H$ be $H+1$.}.
For simplicity we only consider \emph{deterministic and bounded reward function}\footnote{For the sake of presentation, we focus on bounded deterministic reward functions in this work. The techniques can be readily extended to stochastic reward settings and the obtained bounds will match our presented ones.}, which is denoted by $r  = \set{r_1, \dots, r_H}$ where
$r_h:\Scal\times \Acal \to [0,1]$ is the reward function at the $h$-th step.
A \emph{policy} is represented by $\pi := \set{\pi_1,\dots, \pi_H }$, where each $\pi_h: \Scal \to [0,1]^\Acal$ is a potentially random policy at the $h$-th step.
For a policy $\pi$, the \emph{$Q$-value function} and the \emph{value function} are defined as
\begin{align*}
     Q^\pi_h (x,a) 
     &:= \Ebb_{\pi, \Pbb} \bigl[ \sum_{t\ge h} r_t(x_t, a_t) \big| x_h = x, a_h=a \bigr], \\
    V^\pi_h (x) &:= Q^\pi_h (x, \pi_h(x)),
\end{align*}
where the trajectory is given by $x_t \sim \prob{\cdot \mid x_{t-1}, a_{t-1}}$ and $a_t \sim \pi_t(x_t)$ for $t > h$.
The following well-known Bellman equation is worth mentioning:
\begin{align*}
 Q^\pi_h(x,a) = r_h(x,a) + \Ebb_{y\sim\Pbb(\cdot \mid x,a)} V^\pi_{h+1}(y).
\end{align*}
$\pi^* \in \arg\max_{\pi}V_1^\pi (x_1)$ is an optimal policy, and its induced optimal $Q$-value function and the optimal value function are denoted by \(Q^*_h(x, a) := Q^{\pi^*}_h (x, a)\) and \( V^*_h(x) := V^{\pi^*}_h(x)\), respectively.

\paragraph{Sub-Optimality Gap.}
Given an MDP, the \emph{stage-dependent state-action sub-optimality gap}~(see, e.g., \citet{simchowitz2019non}) is defined as
\[\gap_h(x,a) := V^*_h(x) - Q^*_h(x,a) \ge 0.\]
Clearly, $\gap_h(x,a) = 0$ if and only if $a$ is an optimal action at state $x$ and at the $h$-th decision step.
Intuitively, when $\gap_h(x,a) > 0$, $\gap_h(x,a)$ characterizes the difficulty to distinguish the sub-optimal action $a$ from the optimal actions at state $x$ and at the $h$-th step; and the larger $\gap_h(x,a)$ is, the easier should it be distinguishing $a$ from the optimal actions.
The \emph{minimum sub-optimality gap} is then defined as
\begin{equation}\label{eq:min-gap}
\gap_{\min} := \min_{h,x,a} \{\gap_h(x,a) :\ \gap_h(x,a) > 0 \}.
\end{equation}
Intuitively, an MDP with a constant $\gap_{\min}$ is easy to learn since constant number of visitations to an state-action pair suffices to distinguish whether it is optimal~\citep{simchowitz2019non}.

\paragraph{Task-Agnostic Exploration.}
The problem of \emph{task-agnostic exploration} (TAE)~\citep{zhang2020task} involves an MDP environment $(\Scal, \Acal, H, \Pbb, x_1)$ and a set of reward functions 
\[\Rcal  \subset \set{r: [H]\times\Scal\times \Acal \to [0,1]}\] 
that could possibly contain infinitely many reward functions.
An agent first decides exploration policies to collect $K$ trajectories from the environment, in which process reward feedback is not observable.
Then an \emph{oblivious} chooses a reward function $r$ from the candidate set $\Rcal$ and reveals it to the agent
\footnote{In \citet{zhang2020nearly}, $N$ reward functions are selected during planning and only bandit signals are available. For the sake of presentation, we assume there is only $1$ reward function and the agent is provided with the full-information of the reward function in the planning phase. Our results and techniques are ready to be extended to setting of \citet{zhang2020nearly} in a standard manner.}, i.e., $r\in\Rcal$ is selected \emph{independently} from the agent's exploration policy. 
Then the agent needs to compute an $(\epsilon, \delta)$-\emph{probably-approximately-correct} (PAC) policy $\pi$ under the reward function $r$, which means:
\begin{equation}\label{eq:TAE-PAC}
     \text{for an oblivious\footnotemark}\ r,\ \Pbb\{ V^*_1(x_1) - V^{\pi}_1(x_1) > \epsilon \} < \delta,
\end{equation}
where the probability is over the randomness of trajectory collecting in the exploration phase.
The \emph{sample complexity} is measured by the number of trajectories $K$ that the agent needs to collect in the exploration phase to guarantee being $(\epsilon, \delta)$-PAC in planning phase.
\footnotetext{An oblivious reward means the reward is chosen independent of the agent's exploration policy, or equivalently, the reward is chosen prior to the beginning of the exploration phase. For example, a reward chosen uniformly at random from $\Rcal$ is oblivious.}

Here are a few remarks on the TAE setting.
Firstly, note that the reward function $r$ in TAE is oblivious and is not adversarial to the agent's exploration policy.
Due to this non-adversarial nature, TAE can be achieved with a sample complexity $\propto \widetilde{\Ocal}(S)$~\citep{zhang2020task,wu2020accommodating}, which is much cheaper than that required to estimate the transition kernel accurately ($\propto\Ocal(S^2)$, the error is measured by total variation distance).
On the other hand, if the reward function is allowed to be chosen \emph{adversarially} against the agent's exploration policy, the problem is known as \emph{reward-free exploration} (RFE)~\citep{jin2020reward}; the counter part of condition \eqref{eq:TAE-PAC} in RFE reads:
\begin{equation*}
     \text{for any}\ r\in\Rcal,\ \Pbb\set{ V^*_1(x_1) - V^{\pi}_1(x_1) > \epsilon } < \delta,
\end{equation*}
where the probability is over the randomness of trajectory collecting in the exploration phase.
Due to the adversarial nature, a RFE algorithm must estimate the environment with high precision, and a $\propto\Ocal(S^2)$ sample complexity is unavoidable~\citep{jin2020reward}.
In the following paper we will focus on the TAE setting to show a more sample-efficient algorithm for benign TAE problems.
Nonetheless, our algorithm naturally extends to other unsupervised RL settings~\citep{jin2020reward,wu2020accommodating} and the improved sample-efficiency for benign instances also holds (by a standard covering argument on value functions or reward functions). 
We refer the reader to Remark~\ref{remark:rfe} in Section~\ref{sec:discussions} for an example of applying our results in reward-free exploration.

\paragraph{Gap-Dependent Unsupervised RL.}
We now formally state the problem of \emph{gap-dependent task-agnostic exploration} (gap-TAE).
Gap-TAE is a benign TAE instance, in which we assume there is a constant, reward-agnostic gap parameter $\rho$ such that
\begin{equation}\label{eq:gap-parameter}
    0 < \rho \le \gap_{\min}(r)\ \text{for every}\ r\in \Rcal,
\end{equation}
where, with a slightly abuse of notation, $\gap_{\min}(r)$ refers to the minimum sub-optimality gap \eqref{eq:min-gap} induced by reward function $r$;
moreover, the gap parameter $\rho$ is known to the agent.
Our focus is to study whether or not gap-TAE problems can be solved in a faster rate compared to the worst-case TAE problems.

We end this section with the definition of a \emph{clip operator}~\citep{simchowitz2019non}:
\[
\clip_\rho [z] := z\cdot \ind{z \ge \rho}, \text{ for } z \in \Rbb \text{ and } \rho > 0,
\]
which cuts a quantity smaller than $\rho$ to $0$.





\section{AN EFFICIENT ALGORITHM}\label{sec:algorithm}

In this section, we introduce \ALGO for solving gap-TAE in a sample-efficient manner.
The algorithm is formally presented as Algorithms~\ref{alg:exploration} and \ref{alg:planning}, for exploration and planning, respectively.

\begin{algorithm}[htb!]
    \caption{\texttt{\ALGO} (Exploration)}
    \label{alg:exploration}
    \begin{algorithmic}[1]
    \REQUIRE gap parameter $\rho$, number of episodes $K$
    \STATE initialize history $\Hcal^0 = \emptyset$
    \STATE set up a constant $\iota := \log (2HS^2 AK/\delta)$
    \FOR{episode $k=1,2,\dots,K$}
        \STATE $N^k(x,a),\ \widehat{\Pbb}^k (y \mid x,a)\leftarrow \EmpiTrans(\Hcal^{k-1})$
        \STATE compute exploration bonus \( c^k(x,a) := \clip_{\frac{\rho}{2H}}\sbracket{ \sqrt{\frac{ 8H^2 \iota }{N^k(x,a)} }} + \frac{120 (S + H)H^3\iota}{N^k(x,a)} + \frac{240 H^6S^2\iota^2}{\bracket{N^k(x,a)}^2 }\) \label{line:clipped-bonus}
        \STATE $\{ \overline{Q}^k_h(x,a), \overline{V}^k_h(x) \}_{h=1}^H  \leftarrow \UCBQ(\widehat{\Pbb}^{k}, 0, c^k)$\label{line:zero-reward}
        \STATE receive initial state $x_1^k = x_1$
        \FOR{step $h=1,2,\dots,H$}
            \STATE take action $a^k_h = \arg\max_a \overline{Q}^k_h(x^k_h, a)$
            \STATE obtain a new state $x^{k}_{h+1}$
        \ENDFOR 
        \STATE update history $\Hcal^k = \Hcal^{k-1} \cup \{x^k_h, a^k_h\}_{h=1}^H$
    \ENDFOR
    \RETURN History $\Hcal^k$

    \item[]
    
    \STATE \textbf{Function} \EmpiTrans \label{line:empi-trans}
    \STATE \textbf{Require:} history $\Hcal^{k-1}$ 
    \FOR{$(x,a,y) \in \Scal \times \Acal \times \Scal$}
    \STATE \(N^k (x,a,y) := \# \{(x,a,y) \in \Hcal^{k-1} \} \) \\
    \STATE \( N^k(x,a) := \sum_{y} N^k (x,a,y) \)
        \IF{$N^k(x,a) > 0$}
        \STATE $\widehat{\Pbb}^k(y\mid x,a) = {N^k(x,a,y)} / {N^k(x,a)}$
        \ELSE 
        \STATE $\widehat{\Pbb}^k(y\mid x,a) = 1/S$
        \ENDIF
    \ENDFOR
    \RETURN $N^k(x,a), \widehat{\Pbb}^k(y\mid x,a)$

    \item[]

    \STATE \textbf{Function} \UCBQ \label{line:ucbq}
    \STATE \textbf{Require:} empirical transition $\widehat\Pbb^k$, reward function $r$, bonus function $b^k$ 
        \STATE set $V^k_{H+1}(x) = 0$ 
        \FOR{step $h = H, H-1, \dots , 1$}
            \FOR{$(x,a) \in \Scal \times \Acal$}
                \STATE $Q^k_h(x,a) = r_h(x, a) + b^k(x,a) + \widehat{\Pbb}^k_h V^k_{h+1}(x,a)$\label{line:add-bonus}
                \STATE $Q^k_h(x,a) \leftarrow \min\{H,\ Q^k_h(x,a) \}$
                \STATE $V^k_h(x)=\max_{a\in \Acal}Q^k_h(x,a)$
            \ENDFOR
        \ENDFOR
        \RETURN $\big\{ Q^k_h(x,a),\ V^k_h(x) \big\}_{h=1}^H$
    \end{algorithmic}
\end{algorithm}

In the exploration phase, \ALGO maintains an estimated maximum cumulative \emph{bonus} based on the current empirical transition kernel (Algorithm~\ref{alg:exploration}, line~\ref{line:zero-reward}), and explores the environment through executing a greedy policy that maximizes the cumulative bonus.
Note that the exploration bonus at a state-action pair is inversely proportional to the number of visitations to the state-action pair (Algorithm~\ref{alg:exploration}, line~\ref{line:clipped-bonus}), thus the exploration policy is encouraged to pay more visits to the state-actions that have not yet been visited for sufficient times, where its induced bonus is larger.
This design permits \ALGO to dynamically explore the environment.
In the planning phase, \ALGO receives an obliviously/independently revealed reward function, and computes an (averaging of a sequence of) \emph{optimistic} estimation to the optimal value function based on the collected samples (Algorithm~\ref{alg:planning}, line~\ref{line:true-reward}).
The outputted policy equivalently corresponds to the (averaged) optimistic estimated value function (Algorithm~\ref{alg:planning}, line~\ref{line:greedy-policy}).

\begin{algorithm}[htb]
    \caption{\ALGO (Planning)}
    \label{alg:planning}
    \begin{algorithmic}[1]
    \REQUIRE history $\Hcal^K$, reward function $r$
    \STATE set up a constant $\iota := \log (2HS^2 AK/\delta)$
    \FOR{$k=1,2,\dots,K$}
        \STATE $N^k(x,a),\ \widehat{\Pbb}^k (y \mid x,a)\leftarrow \EmpiTrans(\Hcal^{k-1})$
        \STATE compute planning bonus \( b^k(x,a) := \sqrt{\frac{ H^2 \iota }{2 N^k(x,a)} }\)\label{line:hoeffding-bonus}
        \STATE $\{Q^k_h(x,a),\, V^k_h(x)\}_{h=1}^H \leftarrow \UCBQ (\widehat{\Pbb}^{k}, r, b^k)$\label{line:true-reward}
        \STATE infer greedy policy $\pi^k_h (x) = \arg\max_a Q^k_h(x,a) $\label{line:greedy-policy}
    \ENDFOR
    \RETURN $\pi$ drawn uniformly from $\{\pi^1,\dots,\pi^K\}$
    \end{algorithmic}
\end{algorithm}

\ALGO is inspired by two existing model-based algorithms: 
\UCBVI~\citep{azar2017minimax} that achieves minimax optimal sample complexity for supervised RL problems,
and \PFUCB~\citep{wu2020accommodating} that efficiently solves preference-free exploration problems in the context of unsupervised multi-objective RL.
Similar to both \UCBVI and \PFUCB, in the planning phase \ALGO (Algorithm~\ref{alg:planning}) adopts an \emph{upper-confidence-bound} (UCB) type bonus to perform optimistic and model-based planning.
Similar to \PFUCB but different from \UCBVI, \ALGO (Algorithm~\ref{alg:exploration}) explores the unknown environment through a greedy policy that maximizes the cumulative exploration bonus.
Different from either \PFUCB or \UCBVI, \ALGO exploits the gap parameter by \emph{clipping} the exploration bonus (Algorithm~\ref{alg:exploration}, line~\ref{line:clipped-bonus}).

The clipped exploration bonus turns out to be a key ingredient for \ALGO to save samples for gap-TAE problems. 
Specifically, the leading-order term in the UCB-type bonus will be brute-force clipped to near zero when a state-action pair has been visited for sufficiently many times, i.e., $N(x,a) \eqsim H^4 \iota / \rho^2$ (Algorithm~\ref{alg:exploration}, line~\ref{line:clipped-bonus}).
This will cause a sudden decrease of the bonus, and discourages the agent to continue to visit this state-action pair.
Recall that the considered MDP has a sub-optimality gap at least $\rho$, thus $N(x,a) \eqsim H^4 \iota / \rho^2$ amount of samples is already sufficient to distinguish sub-optimal actions from the optimal ones at the state-action pair.
Then by clipping the bonus at such state-action pairs, \ALGO spends less unnecessary visitations to these pairs, and saves opportunities for visiting the state-actions that have not yet been visited sufficiently.
In consequence, \ALGO accelerates benign TAE instances by exploiting the gap parameter.
The above discussions are formally justified in the next section.

\section{THEORETIC RESULTS}\label{sec:theory}
We turn to present our theoretical results. 
We will first show Theorem \ref{thm:gap-upper-bound} that provides a sample-complexity upper bound for the proposed \ALGO algorithm, and Theorem \ref{thm:gap-lower-bound} that provides a sample-complexity lower bound for the gap-TAE problem.
Then we will discuss a novel statistical separation between RL vs. MAB based on these results.

\subsection{Upper and Lower Bounds}
\begin{theorem}[An upper bound for \ALGO]\label{thm:gap-upper-bound}
Suppose that Algorithm~\ref{alg:exploration} accepts a gap parameter $\rho$ that satisfies \eqref{eq:gap-parameter}, and runs for $K$ episodes to collect a dataset $\Hcal^K$.
Let policy $\pi$ be the output of Algorithm~\ref{alg:planning} for an oblivious input reward function $r \in \Rcal$ that is independent of $\Hcal^K$.
Then with probability at least $1-\delta$, the planning error is bounded by
    \begin{align*}
    V^*_1(x_1) - V_1^{\pi}(x_1)
    &\lesssim \frac{H^3SA}{\rho K}\cdot \log\frac{HSAK}{\delta} \\
    &\ + \frac{ H^4 S^2A }{K}\cdot \log(HK) \cdot \log\frac{HSAK}{\delta}.
    \end{align*}
\end{theorem}

\begin{theorem}[A lower bound for gap-TAE]\label{thm:gap-lower-bound}
    Fix $S\ge 5, A\ge 2, H \ge 2+\log_A S$. 
    There exist positive constants $c_1, c_2, \rho_0, \delta_0$, such that for every $\rho \in (0, \rho_0)$, $\epsilon \in (0, \rho)$, $\delta\in (0,\delta_0)$, and for every $(\epsilon, \delta)$-PAC algorithm (see condition \eqref{eq:TAE-PAC}) that runs for $K$ episodes, there exists some gap-TAE instances with a gap parameter $\rho$ that satisfies \eqref{eq:gap-parameter}, such that 
    \[
    \Ebb [K] \ge  c_1\cdot \frac{ H^2 SA}{\rho \epsilon }\cdot\log\frac{c_2}{\delta},   
    \]
    where the expectation is taken with respect to the randomness of choosing the gap-TAE instance.
\end{theorem}

\begin{remark}
According to Theorem~\ref{thm:gap-upper-bound}, \ALGO only requires $\propto\widetilde{\Ocal}(1/\epsilon)$ number of episodes to solve TAE provided with a constant gap parameter, which improves the existing, pessimistic rates $\propto\widetilde{\Ocal}(1/\epsilon^2)$ achieved by algorithms that focus on worst-case TAE instances~\citep{zhang2020task,wang2020reward,wu2020accommodating}.
Moreover, according to Theorem~\ref{thm:gap-lower-bound}, this $\propto\widetilde{\Ocal}(1/\epsilon)$ rate is nearly optimal upto logarithmic factors, which exhibits some fundamental limitations of the acceleration afforded by a constant gap parameter.
A numerical simulation for the acceleration phenomenon is provided in Appendix~\ref{appendix-section:experiment}.
\end{remark}

\begin{remark}
If $\rho \lesssim 1/(HS)$, the error upper bound in Theorem~\ref{thm:gap-upper-bound} is simplified to $\widetilde{\Ocal}(H^3 SA / (\rho K))$, i.e., \ALGO needs at most $K =\widetilde{\Ocal} ( {H^3SA}/{\rho \epsilon} )$ episodes to be $(\epsilon, \delta)$-PAC.
In this regime, Theorem~\ref{thm:gap-lower-bound} suggests that \ALGO achieves a nearly optimal rate in terms of $S$, $A$, $\rho$ and $\epsilon$ (or $K$) ignoring logarithmic factors. Still, the dependence of $H$ is improvable, which we leave as a future work.
\end{remark}


\begin{remark}\label{remark:rfe}
As explained before, our algorithm can be applied to other unsupervised RL settings as well, e.g., reward-free exploration~\citep{jin2020reward} with a gap parameter $\rho$ that satisfies \eqref{eq:gap-parameter}.
In this case, the upper bound in Theorem~\ref{thm:gap-upper-bound} needs to be revised to 
\[
\bigOT{ \frac{H^3 S^2 A}{\rho K} + \frac{H^4 S^3 A}{K} }.
\]
This is obtained by a standard converging and union bound argument on the set of all possible value functions, and from where the extra $S$ factor stem\footnote{For more details, we refer the reader to event \ref{eq:G1} defiend in Appendix \ref{appendix-section:upper-bound-proof}. To apply our result in reward-free exploration, we need \ref{eq:G1} holds for every optimal value function $V_h^*$, which can be guaranteed by covering and union bound argument over all possible value functions, $V_h\in [0,H]^S$.}\footnote{Due to a similar reasoning, when applied to the task-agnostic setting with $N$ selected planning tasks \citep{zhang2020task}, the $\log(HSAK/\delta)$ factor in our bound needs to be modified to $\log(N HSAK/\delta)$.}.
Comparing to the worst-case optimal rate $\propto\Ocal(1/\sqrt{K})$ for RFE \citep{jin2018q,zhang2020nearly}, we again achieve a quadratic saving in terms of $K$ for benign RFE instances with a constant $\rho$.
\end{remark}

\paragraph{Proof Sketch of Theorem~\ref{thm:gap-upper-bound}.}
We first look at the planning phase~(Algorithm~\ref{alg:planning}). Since the reward induces a sub-optimality gap at least $\rho$, with some computations we can obtain the following error estimation of the planning error per episode: for every $k$,
\begin{equation}
        V^*_1(x_1) - V_1^{\pi^k}(x_1) \lesssim \Ebb_{\pi^k, \Pbb} \sum_{h=1}^H c^k(x_h, a_h),
        \label{eq:main:planning-error-bounded-by-sum-of-bonus}
\end{equation}
where $\pi^k$ is the planning policy at the $k$-th episode and $c^k(x,a)$ is the \emph{clipped} exploration bonus at the $k$-th episode.
The right hand side of \eqref{eq:main:planning-error-bounded-by-sum-of-bonus} can be further improved to have a uniform upper bound $H$ per decision step.
Note that the right hand side of \eqref{eq:main:planning-error-bounded-by-sum-of-bonus} is the expected cumulative bonus over the trajectory induced by policy $\pi^k$ and the \emph{true} transition $\Pbb$, and that the bonus contains lower order terms that control the error of an inaccurately estimated probability transition (Algorithm~\ref{alg:exploration}, line~\ref{line:clipped-bonus}).
Therefore, upto some constant factors, it can be bounded by the expected cumulative bonus over the trajectory induced by policy $\pi^k$ and the \emph{empirical} transition $\widehat{\Pbb}^k$.
The above analysis reflects \eqref{eq:main:empi-sum-of-bonus}.
Moreover, \eqref{eq:main:exploration-value} holds naturally according to Algorithm~\ref{alg:exploration}, since $\overline{V}_1^k(x_1)$ maximizes the cumulative bonus over the empirical transition by dynamic programming.
\begin{align}
    V^*_1(x_1) - V_1^{\pi^k}(x_1) 
    &\lesssim \Ebb_{\pi^k, \widehat{\Pbb}^k } \sum_{h=1}^H  H \land  c^k(x_h, a_h) \label{eq:main:empi-sum-of-bonus} \\
    &\le \overline{V}_1^k(x_1). \label{eq:main:exploration-value}
\end{align}
Finally, a standard regret analysis for the exploration phase shows that the total exploration values in the exploration phase is logarithmic, thanks to the clipped exploration bonus. Thus the total planning error is also logarithmic.
The presented error upper bound is established by averaging over the $K$ episodes.
Some of proving techniques are motivated by~\citep{simchowitz2019non,wu2020accommodating}. 
Our key novelty is to carefully incorporate the clip operator in the bonus function and use that to build a connection between the planning and exploration phases.
See Appendix~\ref{appendix-section:upper-bound-proof} for more details.

\begin{figure*}
    \centering
    \includegraphics[width=0.9\linewidth]{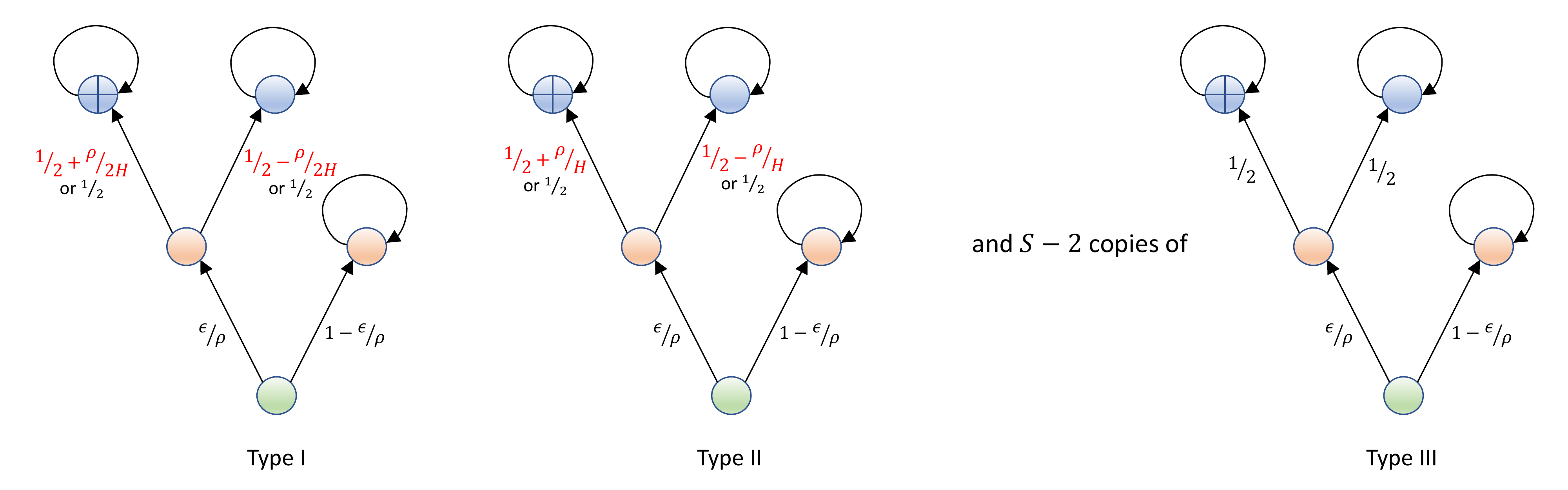}
    \caption{\small A hard-to-learn MDP example. States are denoted by circles. A state-determined reward function takes value $1$ at the state denoted by a circle with a plus sign and takes value zero otherwise. 
    The plot only shows the structure of the last three layers of the MDP, which consists of a Type I model, a Type II model, and $S-2$ Type III model. These models are connected by a $(\log_A S)$-layer tree with $A$-branches in each layer, and with deterministic and known transition.
    In the Type III model, from the green state and for all actions, it transits to the left orange state with probability ${\epsilon}/{\rho}$, or to the self-absorbing right orange state with probability $1-{\epsilon}/{\rho}$.
    Then from the left orange state and for all actions, it transits to the two blue states evenly.
    The Type I (II) model is only different from the Type III model at the left orange state: in Type I (II) model, there exists and only one action such that it transits to the left blue state with probability $1/2 + {\rho}/(2H)$ (with probability $1/2 + {\rho}/{H}$), and to the right blue state otherwise.
    In other words, in the left orange states, an optimal action exists in the Type II model, a second-to-the-best action exists in the Type I model, and all other actions are equivalent and are sub-optimal.
    }
    \label{fig:hard-example}
\end{figure*}

\paragraph{Proof Sketch of Theorem~\ref{thm:gap-lower-bound}.}
A hard instance that witnesses the lower bound is shown in Figure \ref{fig:hard-example}.
The hard instance is motivated by~\citep{mannor2004sample,dann2015sample}. 
One can verify that this instance has a minimum sub-optimality gap $\rho/2$.
Indeed, the only states that have sub-optimal actions are the \emph{left orange states} in Type I model or Type II model. For the left orange state in Type I model, the optimal action has value $H/2 + \rho /2$, but the sub-optimal ones have value $ H / 2$, where the gap is $\rho  / 2$. Similarly we can verify the gap in Type II model is $\rho$.
Moreover, in order to be $\epsilon/3$-correct, the agent needs to identify the optimal action at the left orange states, otherwise it takes a sub-optimal action and incurs a value error at least $\epsilon/ \rho \cdot \rho / 2 = \epsilon / 2$.
Now the problem is reduced to identify the best action at the left orange states. Let us ignore $H, S, A$ factors and focus on $\rho$ and $\epsilon$.
Then at the left orange states, the probability gap between the best action and the second best action is $\propto\Theta(\rho)$, thus $
\propto\Omega(1/\rho^2)$ samples are needed to identify the optimal action.
On the other hand, in each episode there is only $\epsilon / \rho$ chance to visit the left orange states, thus $\propto\Omega(1/(\rho\epsilon))$ episodes are needed to provide $\propto\Omega(1/\rho^2)$ samples at the left orange sates. This justifies the $\propto\Omega(1/(\rho\epsilon))$ rate in the lower bound.
A complete proof is deferred to Appendix~\ref{appendix-section:lower-bound-proof}.

\subsection{Comparison with Multi-Armed Bandit}\label{sec:bandit}
Theorems~\ref{thm:gap-upper-bound} and~\ref{thm:gap-lower-bound} show that for unsupervised RL problems, even when the instance is benign and has a constant gap parameter, a $\propto{\Omega}(1/\epsilon)$ sample complexity must be paid. 
This establishes a stark contrast to the gap-dependent unsupervised exploration problems on a \emph{multi-armed bandit} (MAB). 
In particular, for an MAB with $A$ arms and a minimum sub-optimality gap $\rho$ (that is the gap between the expected rewards of the best action and the second-best action),
a rather simple \emph{uniform exploration} strategy, with $T = \Ocal\big(\frac{A}{\rho^2}\log \frac{A}{\delta}\big) \propto \widetilde{\Ocal}(1)$ pulls of the arms, is $(\epsilon, \delta)$-correct for identifying the best arm (see, e.g., Theorem 33.1 in \citet{lattimore2020bandit}, or Appendix~\ref{appendix-section:bandit-and-mdp-with-generative-model}).
These observations from gap-dependent unsupervised exploration establish an interesting statistical separation between general RL (corresponds to MDP with $H \ge 2$) and MAB (corresponds to MDP with $H =1$), with a sample complexity comparison $\propto\widetilde{\Ocal}(1/\epsilon)$ vs. $\propto\widetilde{\Ocal}(1)$.
This separation suggests that \emph{unsupervised RL is significantly more challenging than unsupervised MAB} even after normalizing the $H$ factor.
The conclusion goes against an emerging wisdom from the supervised RL theory, that RL is nearly as easy as MAB in the supervised setting, given that $H$ factor is normalized~\citep{jiang2018open,wang2020long,zhang2020reinforcement}.


Let us take a deeper look at where RL is harder than MAB.
The hard instance in Figure~\ref{fig:hard-example} clearly illustrates the issue:
\emph{there could exist some important states in MDP that cannot be ignored, but are hard to reach in the same time}, e.g., ignoring the left orange states in Figure \ref{fig:hard-example} would result in a $\Theta(\epsilon)$ error, but there is only $\Theta(\epsilon)$ chance to reach these states per episode, thus in order to reach the left orange states for at least constant times, an algorithm needs at least $\propto\Omega(1/\epsilon)$ number of trajectories.

To further verify our understanding, let us consider an MDP with a sampling simulator~\citep{sidford2018near,sidford2018variance,wang2017randomized,azar2013minimax}. The sampling simulator allows us to draw samples at any state-action pair, thus exempts the ``hard-to-reach'' states.
The following theorem shows that for MDP with a sampling simulator, the gap-TAE problem can also be solved with $\propto\widetilde{\Ocal}(1)$ samples, as in the case of MAB. A proof is included in Appendix \ref{appendix-section:bandit-and-mdp-with-generative-model}.


\begin{theorem}[MDP with a sampling simulator]\label{thm:with-generative-model}
Suppose there is a sampling simulator for the MDP considered in the gap-TAE problem.
Consider exploration with the uniformly sampling strategy, and planning with the dynamic programming method with the obtained empirical probability.
If $T$ samples are drawn, where 
\[
T \ge \frac{2H^4 SA}{\rho^2} \cdot \log \frac{2HSA}{\delta},
\]
then with probability at least $1-\delta$, the obtained policy is optimal ($\epsilon = 0$).
\end{theorem}

\section{ADDITIONAL DISCUSSIONS}\label{sec:discussions}
\paragraph{The Gap Parameter.}
The gap parameter $\rho$ is a hyperparameter for Algorithm~\ref{alg:exploration}.
So long as the inputted gap parameter $\rho$ is a reasonably large constant and satisfies condition \eqref{eq:gap-parameter}, our algorithm achieves a fast, $\propto\widetilde{\Ocal}(1/\epsilon)$ rate for TAE according to Theorem~\ref{thm:gap-upper-bound}.
If the inputted gap parameter $\rho$ violates condition \eqref{eq:gap-parameter}, Theorem~\ref{thm:gap-upper-bound} no longer directly holds; but one can easily show that
\begin{align}
& \ V^*_1(x_1) - V^{\pi}_1(x_1)  = \Ebb_{\pi, \Pbb} \sum_{h=1}^H \big( V^*_h(x_h) - Q^*_h(x_h, a_h) \big) \notag \\
&\le \Ebb_{\pi, \Pbb} \sum_{h=1}^H \clip_{\rho} \sbracket{ V^*_h(x_h) - Q^*_h(x_h, a_h)} + H \rho \label{eq:split-a-gap} \\
&\le \widetilde{\Ocal} \Big( \frac{H^3 SA}{\rho K} + \frac{H^4 S^2 A}{K} + H\rho \Big), \label{eq:apply-bound-to-gap-part}
\end{align}
where \eqref{eq:split-a-gap} is by the definition of clip operator, and \eqref{eq:apply-bound-to-gap-part} is by applying Theorem \ref{thm:gap-upper-bound} to the first term in \eqref{eq:split-a-gap} (that equals to the error on an MDP with a sub-optimality gap $\rho$).
In this way, if the TAE instance is in fact ``hard'' in that condition \eqref{eq:gap-parameter} cannot be met for any reasonably large $\rho$, one can still choose a small $\rho  \eqsim { \epsilon / H}$ and run \ALGO for $K = \widetilde{\Ocal}(H^2 SA / \rho^2 + H^3 S^2 A / \rho)$ episodes to obtain a policy that is $(H\rho, \delta)$-PAC.
Note that this $\propto\widetilde{\Ocal}(1/\epsilon^2)$ rate is minimax optimal ignoring logarithmic factors (and the $H^2$ factor) \citep{jin2020reward,dann2015sample}.
An open question is: is there a sample-efficient algorithm for gap-TAE that does not need an inputted gap parameter?

\paragraph{The $H$ Dependence.}
In the regime that $\rho \lesssim 1/ (HS)$, our upper bound can potentially be improved for an $H$ factor, comparing with the provided lower bound.
Technically, this is because Algorithm~\ref{alg:exploration} utilizes a Hoeffding-type bonus, which is known to be less tight compared with a Bernstein-type bonus~\citep{azar2017minimax}.
Hoeffding-type bonus has the benefits of being \emph{reward-independent}, which allows us to construct a reward-independent, hence unsupervised, exploration policy that minimizes an upper bound of the per-episode, reward-dependent planning error.
In contrast, Bernstein-type bonus is reward-dependent as it requires an (good) estimation to the value function which relies on reward signals.
Since the reward signal is not available in the exploration phase, we do not see a clear way to adopt Bernstein-type bonus in our problem.
We leave the issue of further tightening the $H$ factor as a future work.

\paragraph{Removing the Gap-Independent Term.}
The bound presented in Theorem~\ref{thm:gap-upper-bound} has a gap-independent term, $\widetilde{\Ocal}(H^4 S^2 A / K)$, which is still linear in $\widetilde{\Ocal}(1/K)$ but is quadratic in $S$.
This $S^2$-dependence appears in the sample-complexity lower-order term of all the model-based algorithms that we are aware of, e.g., ~\citep{azar2017minimax,zanette2019tighter,simchowitz2019non}, and could potentially be mitigated by model-free algorithms~\citep{jin2018q,yang2020q}.

\paragraph{Visiting Ratio Based Approaches.}
In the context of reward-free exploration, \citet{zhang2020nearly} extend a visiting ratio based approach that is initially proposed by~\citet{jin2020reward}, and achieves a nearly minimax optimal sample complexity for reward-free exploration.
The visiting ratio based method has the advantage of supporting plug-in planners (i.e., an approximate MDP solver given trasition matrix and reward) \citep{jin2020reward}.
Such approach can be adapted to obtain a $\widetilde{\Ocal}(1/\epsilon)$ fast rate for gap-TAE as well; 
however, it involves a $S^2$-dependence on the obtained bound, which is sub-optimal in the context of task-agnostic exploration (see \citet{zhang2020task,wu2020accommodating} and Theorem~\ref{thm:gap-lower-bound}).

\paragraph{A Refined Gap Dependence?}
For supervised RL, the finest gap-dependent sample complexity bound can accurately characterize the role of each state-action gap \citep{simchowitz2019non,xu2021fine}.
In contrast, for unsupervised RL, our derived gap-dependent bound is stated as a function of a gap parameter that reflects only the \emph{minimum} nonzero state-action gap.
We argue that obtaining a refined gap dependence is generally not
possible in the unsupervised setting.
In particular, let us consider a gap-TAE problem with two candidate rewards, where one of them is specified as in our lower
bound construction (see Theorem \ref{thm:gap-lower-bound} and Figure \ref{fig:hard-example}) and the other induces better state-action gaps;
in the planning phase, one of the two rewards is selected uniformly at
random.
Then any algorithm for this problem instance will need to solve the hard instance in Theorem \ref{thm:gap-lower-bound} with probability $1/2$, therefore it must suffers from the ``worst'' gap parameter instead of the better state-action gaps.

\section{RELATED WORKS}\label{sec:related}

\paragraph{Supervised RL with Gap Dependence.}
In the literature of supervised RL, the state-action sub-optimality gap has been long embraced to characterize instance-dependent theoretic guarantees. \citet{ortner2007logarithmic,tewari2007optimistic,ok2018exploration} study gap-dependent bounds in the asymptotic sense, and \citet{jaksch2010near} establish a finite time, gap-dependent bound.
More recently, in the setting of finite-horizon MDP and for a broad class of UCB-type, model based algorithms, \citet{simchowitz2019non} provide a finite, gap-dependent regret bound that comprehensively interpolates the minimax, gap-independent regret bound and a logarithmic, gap-dependent regret bound.
Similar results are further built for model-free algorithms~\citep{yang2020q}, linear MDP~\citep{he2020logarithmic}, and MDP with corruptions~\citep{lykouris2019corruption}.
More recently, \citet{dann2021beyond} improves the gap-dependent bounds in \citet{simchowitz2019non} by considering the reachability of state-action pairs.
In addition to this line, \citet{jonsson2020planning} study the gap-dependent bound for a Monte-Carlo tree search algorithm, \citet{xu2021fine} establish a fine-grained gap-dependent bound through a non-UCB type algorithm, and \citet{al2021navigating,wagenmaker2021beyond} investigate gap-dependent bound for the best-policy-identification problem\footnote{Similarly to our lower bound construction, \citet{wagenmaker2021beyond} also exploit a transition probability to show their gap-dependent lower bound. We emphasize that our works are concurrent and the similar idea is used for solving different problems.}.
However, all these results explores the environment under the guidance of an observable reward signal, thus are not applicable to the unsupervised exploration problems studied in this paper.

\paragraph{Unsupervised RL in the Worst Case.}
The arguably most typical unsupervised exploration problem is the reward-free exploration problem formalized by \citet{jin2020reward}, where the agent collects samples in the unsupervised fashion in order to be able to plan nearly optimally for arbitrary rewards. 
Prior to \citet{jin2020reward}, \citet{hazan2018provably,brafman2002r,du2019provablyefficient} also study exploratory policies with certain covering properties; 
and following \citet{jin2020reward}, the sample complexity of reward-free exploration is further improved to nearly minimax optimal by \citet{kaufmann2020adaptive,menard2020fast,zhang2020nearly}.
Besides reward-free exploration problems, \citet{zhang2020task,wang2020reward} introduce and study task-agnostic exploration problems where the planning reward is fixed but unknown during exploration, and \citet{wu2020accommodating} study preference-free exploration problems in the context of multi-objective RL.
Nonetheless, the above considerations of unsupervised exploration problems and their algorithms take no advantage of a gap parameter, and the obtained theoretic results are pessimistic and restricted by the worst cases.

\paragraph{Lenient Regret.}
Our results can also be understood from the viewpoint of \emph{lenient error}~\citep{merlis2020lenient}, which is first introduced in the context of MAB and aims to capture a regret that tolerances small errors per decision step.
This notion naturally generalizes to RL as follows:
\begin{equation}\label{eq:lenient-error}
\begin{aligned}
    & \mathtt{LenientError}(\pi) := \\
    &\qquad \Ebb_{\pi, \Pbb} \sum_{h=1}^H \clip_{\rho} \sbracket{ V^*_h(x_h) - Q^*_h(x_h, a_h)}.
\end{aligned}
\end{equation}
Clearly, solving a gap-TAE problem with a gap parameter $\rho$ under the error measured by $V^*_1(x_1) - V^\pi_1(x_1)$ is equivalent to solving a TAE problem under a lenient error measured by \eqref{eq:lenient-error}.
From this viewpoint, an interesting future direction to extend our results to general gap function (see Definition 1 in \citet{merlis2020lenient}) beyond the $\clip_\rho [\cdot]$ studied in this work.

\section{CONCLUSION}\label{sec:conclusion}
In this paper we study sample-efficient algorithms for gap-dependent unsupervised exploration problems in RL. When the targeted planning tasks have a constant gap parameter, the proposed algorithm achieves a gap-dependent sample complexity upper bound that significantly improves the existing pessimistic bounds. 
Moreover, an information-theoretic lower bound is provided to justify the tightness of the obtained upper bound. 
These results establish an interesting statistical separation between RL and MAB (or RL with a simulator) in terms of gap-dependent unsupervised exploration problems.

\section*{Acknowledgements}

We would like to thank the anonymous reviewers and area chairs for their helpful comments. 
This research was supported by the Defense Advanced Research Projects Agency (DARPA) under Contract No. HR00112190130.
\bibliography{ref}

\begin{thebibliography}{}

\bibitem[Al~Marjani et~al., 2021]{al2021navigating}
Al~Marjani, A., Garivier, A., and Proutiere, A. (2021).
\newblock Navigating to the best policy in markov decision processes.
\newblock {\em Advances in Neural Information Processing Systems}, 34.

\bibitem[Azar et~al., 2013]{azar2013minimax}
Azar, M.~G., Munos, R., and Kappen, H.~J. (2013).
\newblock Minimax pac bounds on the sample complexity of reinforcement learning
  with a generative model.
\newblock {\em Machine learning}, 91(3):325--349.

\bibitem[Azar et~al., 2017]{azar2017minimax}
Azar, M.~G., Osband, I., and Munos, R. (2017).
\newblock Minimax regret bounds for reinforcement learning.
\newblock In {\em Proceedings of the 34th International Conference on Machine
  Learning-Volume 70}, pages 263--272. JMLR. org.

\bibitem[Brafman and Tennenholtz, 2002]{brafman2002r}
Brafman, R.~I. and Tennenholtz, M. (2002).
\newblock R-max-a general polynomial time algorithm for near-optimal
  reinforcement learning.
\newblock {\em Journal of Machine Learning Research}, 3(Oct):213--231.

\bibitem[Dann and Brunskill, 2015]{dann2015sample}
Dann, C. and Brunskill, E. (2015).
\newblock Sample complexity of episodic fixed-horizon reinforcement learning.
\newblock In {\em Advances in Neural Information Processing Systems}, pages
  2818--2826.

\bibitem[Dann et~al., 2017]{dann2017unifying}
Dann, C., Lattimore, T., and Brunskill, E. (2017).
\newblock Unifying pac and regret: Uniform pac bounds for episodic
  reinforcement learning.
\newblock {\em arXiv preprint arXiv:1703.07710}.

\bibitem[Dann et~al., 2021]{dann2021beyond}
Dann, C., Marinov, T.~V., Mohri, M., and Zimmert, J. (2021).
\newblock Beyond value-function gaps: Improved instance-dependent regret bounds
  for episodic reinforcement learning.
\newblock {\em Advances in Neural Information Processing Systems}, 34.

\bibitem[Du et~al., 2019]{du2019provablyefficient}
Du, S., Krishnamurthy, A., Jiang, N., Agarwal, A., Dudik, M., and Langford, J.
  (2019).
\newblock Provably efficient rl with rich observations via latent state
  decoding.
\newblock In {\em International Conference on Machine Learning}, pages
  1665--1674. PMLR.

\bibitem[Finn and Levine, 2017]{finn2017deep}
Finn, C. and Levine, S. (2017).
\newblock Deep visual foresight for planning robot motion.
\newblock In {\em 2017 IEEE International Conference on Robotics and Automation
  (ICRA)}, pages 2786--2793. IEEE.

\bibitem[Hazan et~al., 2018]{hazan2018provably}
Hazan, E., Kakade, S.~M., Singh, K., and Van~Soest, A. (2018).
\newblock Provably efficient maximum entropy exploration.
\newblock {\em arXiv preprint arXiv:1812.02690}.

\bibitem[He et~al., 2020]{he2020logarithmic}
He, J., Zhou, D., and Gu, Q. (2020).
\newblock Logarithmic regret for reinforcement learning with linear function
  approximation.
\newblock {\em arXiv preprint arXiv:2011.11566}.

\bibitem[Jaksch et~al., 2010]{jaksch2010near}
Jaksch, T., Ortner, R., and Auer, P. (2010).
\newblock Near-optimal regret bounds for reinforcement learning.
\newblock {\em Journal of Machine Learning Research}, 11(Apr):1563--1600.

\bibitem[Jiang and Agarwal, 2018]{jiang2018open}
Jiang, N. and Agarwal, A. (2018).
\newblock Open problem: The dependence of sample complexity lower bounds on
  planning horizon.
\newblock In {\em Conference On Learning Theory}, pages 3395--3398. PMLR.

\bibitem[Jin et~al., 2018]{jin2018q}
Jin, C., Allen-Zhu, Z., Bubeck, S., and Jordan, M.~I. (2018).
\newblock Is q-learning provably efficient?
\newblock In {\em Advances in Neural Information Processing Systems}, pages
  4863--4873.

\bibitem[Jin et~al., 2020]{jin2020reward}
Jin, C., Krishnamurthy, A., Simchowitz, M., and Yu, T. (2020).
\newblock Reward-free exploration for reinforcement learning.
\newblock {\em arXiv preprint arXiv:2002.02794}.

\bibitem[Jonsson et~al., 2020]{jonsson2020planning}
Jonsson, A., Kaufmann, E., M{\'e}nard, P., Domingues, O.~D., Leurent, E., and
  Valko, M. (2020).
\newblock Planning in markov decision processes with gap-dependent sample
  complexity.
\newblock {\em arXiv preprint arXiv:2006.05879}.

\bibitem[Kaufmann et~al., 2020]{kaufmann2020adaptive}
Kaufmann, E., M{\'e}nard, P., Domingues, O.~D., Jonsson, A., Leurent, E., and
  Valko, M. (2020).
\newblock Adaptive reward-free exploration.
\newblock {\em arXiv preprint arXiv:2006.06294}.

\bibitem[Lattimore and Szepesv{\'a}ri, 2020]{lattimore2020bandit}
Lattimore, T. and Szepesv{\'a}ri, C. (2020).
\newblock {\em Bandit algorithms}.
\newblock Cambridge University Press.

\bibitem[Lykouris et~al., 2019]{lykouris2019corruption}
Lykouris, T., Simchowitz, M., Slivkins, A., and Sun, W. (2019).
\newblock Corruption robust exploration in episodic reinforcement learning.
\newblock {\em arXiv preprint arXiv:1911.08689}.

\bibitem[Mannor and Tsitsiklis, 2004]{mannor2004sample}
Mannor, S. and Tsitsiklis, J.~N. (2004).
\newblock The sample complexity of exploration in the multi-armed bandit
  problem.
\newblock {\em Journal of Machine Learning Research}, 5(Jun):623--648.

\bibitem[Maurer and Pontil, 2009]{maurer2009empirical}
Maurer, A. and Pontil, M. (2009).
\newblock Empirical bernstein bounds and sample variance penalization.
\newblock {\em arXiv preprint arXiv:0907.3740}.

\bibitem[M{\'e}nard et~al., 2020]{menard2020fast}
M{\'e}nard, P., Domingues, O.~D., Jonsson, A., Kaufmann, E., Leurent, E., and
  Valko, M. (2020).
\newblock Fast active learning for pure exploration in reinforcement learning.
\newblock {\em arXiv preprint arXiv:2007.13442}.

\bibitem[Merlis and Mannor, 2020]{merlis2020lenient}
Merlis, N. and Mannor, S. (2020).
\newblock Lenient regret for multi-armed bandits.
\newblock {\em arXiv preprint arXiv:2008.03959}.

\bibitem[Ok et~al., 2018]{ok2018exploration}
Ok, J., Proutiere, A., and Tranos, D. (2018).
\newblock Exploration in structured reinforcement learning.
\newblock {\em arXiv preprint arXiv:1806.00775}.

\bibitem[Ortner and Auer, 2007]{ortner2007logarithmic}
Ortner, P. and Auer, R. (2007).
\newblock Logarithmic online regret bounds for undiscounted reinforcement
  learning.
\newblock {\em Advances in Neural Information Processing Systems}, 19:49.

\bibitem[Riedmiller et~al., 2018]{riedmiller2018learning}
Riedmiller, M., Hafner, R., Lampe, T., Neunert, M., Degrave, J., Wiele, T.,
  Mnih, V., Heess, N., and Springenberg, J.~T. (2018).
\newblock Learning by playing solving sparse reward tasks from scratch.
\newblock In {\em International Conference on Machine Learning}, pages
  4344--4353. PMLR.

\bibitem[Schaul et~al., 2015]{schaul2015universal}
Schaul, T., Horgan, D., Gregor, K., and Silver, D. (2015).
\newblock Universal value function approximators.
\newblock In {\em International conference on machine learning}, pages
  1312--1320.

\bibitem[Sidford et~al., 2018a]{sidford2018near}
Sidford, A., Wang, M., Wu, X., Yang, L., and Ye, Y. (2018a).
\newblock Near-optimal time and sample complexities for solving markov decision
  processes with a generative model.
\newblock In {\em Advances in Neural Information Processing Systems}, pages
  5186--5196.

\bibitem[Sidford et~al., 2018b]{sidford2018variance}
Sidford, A., Wang, M., Wu, X., and Ye, Y. (2018b).
\newblock Variance reduced value iteration and faster algorithms for solving
  markov decision processes.
\newblock In {\em Proceedings of the Twenty-Ninth Annual ACM-SIAM Symposium on
  Discrete Algorithms}, pages 770--787. SIAM.

\bibitem[Simchowitz and Jamieson, 2019]{simchowitz2019non}
Simchowitz, M. and Jamieson, K. (2019).
\newblock Non-asymptotic gap-dependent regret bounds for tabular mdps.
\newblock {\em arXiv preprint arXiv:1905.03814}.

\bibitem[Tewari and Bartlett, 2007]{tewari2007optimistic}
Tewari, A. and Bartlett, P.~L. (2007).
\newblock Optimistic linear programming gives logarithmic regret for
  irreducible mdps.
\newblock In {\em NIPS}, pages 1505--1512. Citeseer.

\bibitem[Wagenmaker et~al., 2021]{wagenmaker2021beyond}
Wagenmaker, A., Simchowitz, M., and Jamieson, K. (2021).
\newblock Beyond no regret: Instance-dependent pac reinforcement learning.
\newblock {\em arXiv preprint arXiv:2108.02717}.

\bibitem[Wang, 2017]{wang2017randomized}
Wang, M. (2017).
\newblock Randomized linear programming solves the discounted markov decision
  problem in nearly-linear (sometimes sublinear) running time.
\newblock {\em arXiv preprint arXiv:1704.01869}.

\bibitem[Wang et~al., 2020a]{wang2020long}
Wang, R., Du, S.~S., Yang, L.~F., and Kakade, S.~M. (2020a).
\newblock Is long horizon reinforcement learning more difficult than short
  horizon reinforcement learning?
\newblock {\em arXiv preprint arXiv:2005.00527}.

\bibitem[Wang et~al., 2020b]{wang2020reward}
Wang, R., Du, S.~S., Yang, L.~F., and Salakhutdinov, R. (2020b).
\newblock On reward-free reinforcement learning with linear function
  approximation.
\newblock {\em arXiv preprint arXiv:2006.11274}.

\bibitem[Wu et~al., 2021]{wu2020accommodating}
Wu, J., Braverman, V., and Yang, L.~F. (2021).
\newblock Accommodating picky customers: Regret bound and exploration
  complexity for multi-objective reinforcement learning.
\newblock {\em The 35th Conference on Neural Information Processing Systems}.

\bibitem[Xie et~al., 2019]{xie2019improvisation}
Xie, A., Ebert, F., Levine, S., and Finn, C. (2019).
\newblock Improvisation through physical understanding: Using novel objects as
  tools with visual foresight.
\newblock {\em arXiv preprint arXiv:1904.05538}.

\bibitem[Xie et~al., 2018]{xie2018few}
Xie, A., Singh, A., Levine, S., and Finn, C. (2018).
\newblock Few-shot goal inference for visuomotor learning and planning.
\newblock In {\em Conference on Robot Learning}, pages 40--52. PMLR.

\bibitem[Xu et~al., 2021]{xu2021fine}
Xu, H., Ma, T., and Du, S.~S. (2021).
\newblock Fine-grained gap-dependent bounds for tabular mdps via adaptive
  multi-step bootstrap.
\newblock {\em arXiv preprint arXiv:2102.04692}.

\bibitem[Yang et~al., 2020]{yang2020q}
Yang, K., Yang, L.~F., and Du, S.~S. (2020).
\newblock $ q $-learning with logarithmic regret.
\newblock {\em arXiv preprint arXiv:2006.09118}.

\bibitem[Zanette and Brunskill, 2019]{zanette2019tighter}
Zanette, A. and Brunskill, E. (2019).
\newblock Tighter problem-dependent regret bounds in reinforcement learning
  without domain knowledge using value function bounds.
\newblock In {\em International Conference on Machine Learning}, pages
  7304--7312. PMLR.

\bibitem[Zhang et~al., 2020a]{zhang2020task}
Zhang, X., Singla, A., et~al. (2020a).
\newblock Task-agnostic exploration in reinforcement learning.
\newblock {\em arXiv preprint arXiv:2006.09497}.

\bibitem[Zhang et~al., 2020b]{zhang2020nearly}
Zhang, Z., Du, S.~S., and Ji, X. (2020b).
\newblock Nearly minimax optimal reward-free reinforcement learning.
\newblock {\em arXiv preprint arXiv:2010.05901}.

\bibitem[Zhang et~al., 2020c]{zhang2020reinforcement}
Zhang, Z., Ji, X., and Du, S.~S. (2020c).
\newblock Is reinforcement learning more difficult than bandits? a near-optimal
  algorithm escaping the curse of horizon.
\newblock {\em arXiv preprint arXiv:2009.13503}.

\end{thebibliography}

\newpage
\appendix
\onecolumn

\section{NUMERICAL SIMULATIONS}\label{appendix-section:experiment}
Figure~\ref{fig:rate} illustrates the fast rate achieved by \ALGO for task-agnostic exploration on a benign MDP with constant minimum sub-optimality gap. The curve for \ALGO indicates the planning error of \ALGO when running on a random MDP with $H=5, S=10, A=10, \rho=0.4$ and $K=50,000$.
By comparing with the minimax rate, we observe that \ALGO solves task-agnostic exploration with a faster rate, when the task has a constant minimum sub-optimality gap.

In this experiment, the random MDP can be regarded as a Grid World. 
It is generated as follows. 
The reward is $1$ at a state $x^*$ and is $0$ otherwise. The initial state is fixed $x_0 \neq x^*$ and cannot be revisited in the rest steps in one game. 
$x^*$ can only be reached by taking an optimal action $a^*$ at two states, $x_0$ and $x^*$, where $\Pbb \{ x^* \vert x_0, a^*\} = \rho$ and $\Pbb\{ x^* \vert x^*, a^* \} = 1$.
Except for these constraints, the transition kernel is filled with random numbers uniformly distributed in $(0,1)$ and is then properly normalized.
This MDP has a gap parameter $\rho$. 
In experiment the exploration bonus in Algorithm \ref{alg:exploration} is simplified: only the first two terms are considered (which are leading terms) and absolute constants and logarithmic factors are set to be $1$. 

\begin{figure}
    \centering
    \includegraphics[width=0.6\linewidth]{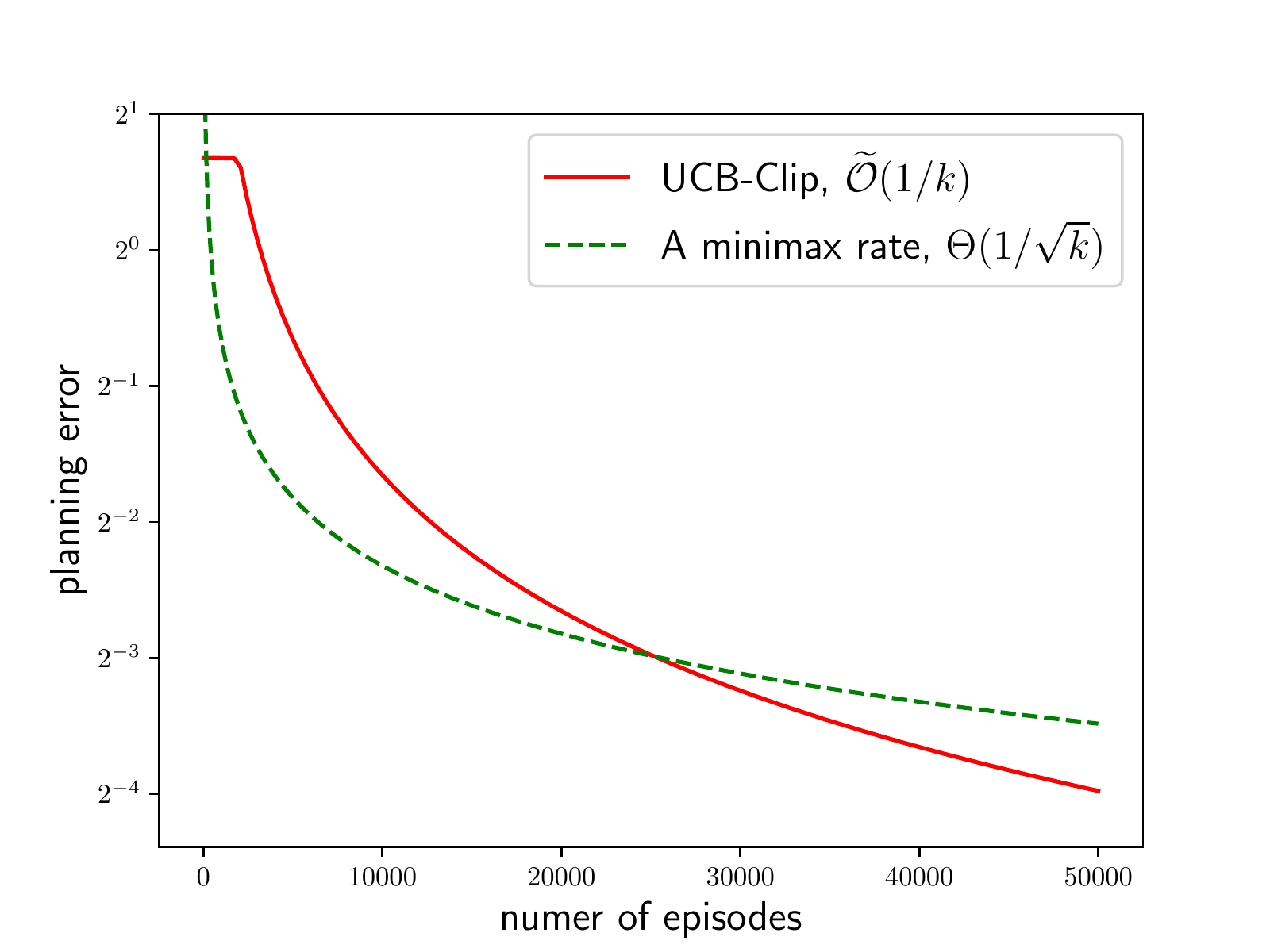}
    \caption{\small An illustration of the fast task-agnostic exploration achieved by \ALGO. In the experiment we simulate a random MDP with $H=5, S=10, A=10$ and $\rho=0.4$, and run \ALGO for $K=50,000$ episodes. The red curve shows the planning error of \ALGO, and the green dotted curve shows a minimax error rate. The plot shows that \ALGO achieves an improved rate for gap-dependent task-agnostic exploration.}
    \label{fig:rate}
\end{figure}

\section{PROOF OF THE UPPER BOUND (THEOREM \ref{thm:gap-upper-bound})}\label{appendix-section:upper-bound-proof}
Our proof is inspired by \citep{simchowitz2019non} and \citep{wu2020accommodating}.

\paragraph{Preliminaries.}
Let $\pi^k$ be the planning policy at the $k$-th episode, i.e., a greedy policy that maximizes $Q^k_h(x,a)$.
Let $\bar{\pi}^k$ be the exploration policy at the $k$-th episode, i.e., a greedy policy given that maximizes $\overline{Q}^k_h(x,a)$. 
Let $w^k_h(x,a) := \Pbb \set{(x_h, a_h) = (x,a) \mid \bar{\pi}^k, \Pbb}$ and $w^k(x,a) := \sum_{h} w^k_h(x,a) $.
In the following, if not otherwise noted, we define 
\( \iota := \log \frac{2HS^2AK}{\delta}. \)

Consider the following good events 
\begingroup
\allowdisplaybreaks
\begin{gather}
    G_1 := \set{ \forall x,a,h,k,\ \abs{(\widehat{\Pbb}^k - \Pbb)V^*_{h+1}(x,a)} \le \sqrt{\frac{H^2}{2N^k(x,a)}\log\frac{2HSAK}{\delta} } }, \label{eq:G1} \tag{$G_1$}\\
    \begin{aligned}
        G_2 &:= \Bigg\{\forall x,a,y,k,\  \abs{\widehat{\Pbb}^k(y\mid x, a) - \Pbb(y\mid x, a)} \le \\
        &\qquad\qquad\qquad \sqrt{\frac{2\Pbb(y\mid x, a) }{N^k(x,a)}\log\frac{2S^2AK}{\delta}} 
        + \frac{2}{3N^k(x,a)}\log\frac{2S^2AK}{\delta}   \Bigg\}, 
    \end{aligned}\label{eq:G2} \tag{$G_2$}\\
    \begin{aligned}
        G_3 &:= \Bigg\{\forall x,a,y,k,\  \abs{\widehat{\Pbb}^k(y\mid x, a) - \Pbb(y\mid x, a)} \le \\
        &\qquad\qquad\qquad \sqrt{\frac{2\widehat{\Pbb}^k(y\mid x, a) }{N^k(x,a)}\log\frac{2S^2AK}{\delta}} 
        + \frac{7}{3N^k(x,a)}\log\frac{2S^2AK}{\delta}   \Bigg\}, 
    \end{aligned}\label{eq:G3} \tag{$G_3$}\\
    G_4 := \set{\forall x,a,k,\ N^k(x,a) \ge \half \sum_{j < k} w^j(x,a) - H \log \frac{HSA}{\delta} }. \label{eq:G4} \tag{$G_4$}
\end{gather}
\endgroup

\begin{lemma}[The probability of good events]\label{lemma:good-probability}
    $\Pbb\set{ G_1 \cap G_2 \cap G_3 \cap G_4 } \ge 1-4\delta$.
\end{lemma}
\begin{proof}
    By Hoeffding's inequality and a union bound, we have that $\Pbb \set{G_1} \ge 1-\delta$.

    By Bernstein's inequality, a union  and that $1-\Pbb(y \mid x,a) \le 1$, we have that $\Pbb \set{G_2} \ge 1-\delta$.

    By empirical Bernstein's inequality \citep{maurer2009empirical}, a union bound and that $1-\widehat{\Pbb}^k(y \mid x,a) \le 1$, we have that $\Pbb \set{G_3} \ge 1-\delta$.

    According to Lemma F.4 by \citep{dann2017unifying} and a union bound, we have that $\Pbb\set{G_4} \ge 1-\delta$.
    
    Finally, a union bound over the four events proves the claim.
\end{proof}

\paragraph{Planning Phase.}
Recall the planning bonus is set to be 
\begin{equation}\label{eq:planning-bonus}
    b^k(x,a) := \sqrt{\frac{H^2\iota}{2N^k(x,a)}}.
\end{equation}

\begin{lemma}[Optimistic planning]\label{lemma:optimistic-planning}
    If $G_1$ holds, then \( Q^k_h(x,a) \ge Q^*_h(x,a) \) for every $k,h,x,a$.
\end{lemma}
\begin{proof}
    We prove it by induction. 
    Clearly the hypothesis holds for $H+1$; now suppose that $Q^k_{h+1}(x,a) \ge Q^*_{h+1}(x,a)$, and consider $h$.
    From Algorithm \ref{alg:planning}, we see that 
        \begin{equation}\label{eq:Qk}
            Q^k_h(x,a) := H\land \bracket{r_h(x,a) + b^k(x,a) + \widehat{\Pbb}^k  V^k_{h+1} (x,a)}.
        \end{equation}
        If $Q^k_h(x,a) = H$, then $Q^k_h(x,a) = H \ge Q^*_h(x,a)$; otherwise, we have that
        \begin{align*}
            Q^k_h(x,a) - Q^*_h(x,a)
            &= r_h(x,a) + b^k(x,a) + \widehat{\Pbb}^k  V^k_{h+1} (x,a) - r_h(x,a) - \Pbb V^*_{h+1} (x,a) \\
            &= b^k(x,a) + \widehat{\Pbb}^k V^k_{h+1} (x,a) - \Pbb V^*_{h+1} (x,a) \\
            &\ge b^k(x,a) + (\widehat{\Pbb}^k - \Pbb) V^*_{h+1} (x,a) \qquad (\text{since $Q^k_{h+1}(x,a) \ge Q^*_{h+1}(x,a)$})\\
            &\ge 0 \qquad (\text{since $G_1$ holds}).
        \end{align*}
        These complete our induction.
\end{proof}

Let us denote the \emph{optimistic surplus} \citep{simchowitz2019non} as 
\begin{equation}\label{eq:optimistic-surplus}
    E^k_h(x,a) := Q^k_h(x,a) - \bracket{r_h(x,a) + {\Pbb V^k_{h+1}}(x,a) }.
\end{equation}
\begin{lemma}[Optimistic surplus bound]\label{lemma:surplus-bound}
    If $G_1$, $G_2$ and $G_3$ hold, then for every $k,h,x,a$,
    \[
            E^k_h(x,a) \le H \land \bracket{ \sqrt{\frac{2 H^2\iota}{N^k(x,a)}} + \frac{H S\iota}{N^k(x,a)} + \Ebb_{\pi^k, \Pbb} \sum_{t\ge h+1} \bracket{ {\frac{4 e^2 H^3\iota}{N^k(x_t, a_t)}} +  \frac{8 e^2 H^5S^2\iota^2}{\bracket{N^k(x_t, a_t) }^2} } }.
            \]
\end{lemma}
\begin{proof}
    By \eqref{eq:optimistic-surplus} and \eqref{eq:Qk} we have $E^k_h(x,a) \le Q^k_h(x,a) \le H$.
    As for the second bound, note that
    \begin{align}
        E^k_h(x,a) 
        &= Q^k_h(x,a) - \bracket{r_h(x,a) + \Pbb V^k_{h+1}(x,a) } \qquad (\text{use \eqref{eq:optimistic-surplus}}) \notag \\
        &\le r_h(x,a) + b^k(x,a) + \widehat{\Pbb}^k  V^k_{h+1} (x,a) - r_h(x,a) - \Pbb V^k_{h+1} (x,a) \qquad (\text{use \eqref{eq:Qk}}) \notag \\
        &= b^k(x,a) + (\widehat{\Pbb}^k - \Pbb) V^*_{h+1}(x,a)  + (\widehat{\Pbb}^k - \Pbb) (V^k_{h+1} - V^*_{h+1}) (x,a) \notag \\
        &\le 2 b^k(x,a) +  (\widehat{\Pbb}^k - \Pbb) (V^k_{h+1} - V^*_{h+1}) (x,a) \qquad (\text{use \eqref{eq:planning-bonus} and that $G_1$ holds})\notag \\
        &\le \sqrt{\frac{2H^2\iota}{N^k(x,a)}}  + \frac{ H S\iota}{N^k(x,a)} +  \Pbb\bracket{V^k_{h+1} - V^*_{h+1}}^2 (x,a). \quad (\text{use Lemma \ref{lemma:lower-order-term}}) \label{eq:Ek-first-bound}
    \end{align}
We next bound $V^k_{h}(x) - V^*_{h}(x)$ by
\begingroup
\allowdisplaybreaks
\begin{align*}
    & \ V^k_{h}(x) - V^*_{h}(x)
    \le Q^k_h(x,a) - Q^*_h(x,a) \qquad(\text{set $a=\pi^k(x)$})\\
    &\le b^k(x,a) +  \widehat{\Pbb}^k V^k_{h+1} (x,a)  -  {\Pbb} V^*_{h+1} (x,a) \qquad(\text{use \eqref{eq:Qk}}) \\
    &= \Big(  b^k +  {\Pbb} \bracket{V^k_{h+1} -  V^*_{h+1}} +  (\widehat{\Pbb}^k - \Pbb )\bracket{V^k_{h+1} -  V^*_{h+1}}  + (\widehat{\Pbb}^k - \Pbb ){ V^*_{h+1}} \Big)(x,a) \\
    &\le \Big( 2b^k + {\Pbb} \bracket{V^k_{h+1} -  V^*_{h+1}} +  (\widehat{\Pbb}^k - \Pbb ) \bracket{V^k_{h+1} -  V^*_{h+1}}  \Big)(x,a) \quad (\text{use \eqref{eq:planning-bonus} and $G_1$})\\
    &\le  \sqrt{\frac{2H^2\iota}{N^k(x,a)}} + \bracket{1+\frac{1}{H}}{\Pbb} \bracket{V^k_{h+1} -  V^*_{h+1}}(x,a) +  \frac{2H^2S\iota}{N^k(x,a)}  . \qquad (\text{use Lemma \ref{lemma:lower-order-term}})
\end{align*}
\endgroup
Solving the recursion we obtain 
\begin{align*}
    V^k_{h}(x) - V^*_{h}(x) \le e \cdot \Ebb_{\pi^k, \Pbb} \sum_{t\ge h} \bracket{ \sqrt{\frac{2 H^2\iota}{N^k(x_t, a_t)}} +  \frac{2 H^2S\iota}{N^k(x_t, a_t)} },
\end{align*}
where $x_h = x$.
This implies that 
\begin{align*}
    \bracket{V^k_{h}(x) - V^*_{h}(x) }^2 
    &\le \bracket{\Ebb_{\pi^k, \Pbb} \sum_{t\ge h} \bracket{ \sqrt{\frac{2e^2 H^2\iota}{N^k(x_t, a_t)}} +  \frac{2e H^2S\iota}{N^k(x_t, a_t)} } }^2 \\
    &\le  \Ebb_{\pi^k, \Pbb}  \bracket{ \sum_{t\ge h}\sqrt{\frac{2e^2 H^2\iota}{N^k(x_t, a_t)}} +  \frac{2e H^2S\iota}{N^k(x_t, a_t)} }^2 \qquad(\text{$(\cdot)^2$ is convex})\\
    &\le H \cdot \Ebb_{\pi^k, \Pbb} \sum_{t\ge h} \bracket{ \sqrt{\frac{2e^2 H^2\iota}{N^k(x_t, a_t)}} +  \frac{2e H^2S\iota}{N^k(x_t, a_t)} }^2 \ (\text{Cauchy–Schwarz})\\
    &\le \Ebb_{\pi^k, \Pbb} \sum_{t\ge h} \bracket{ {\frac{4e^2 H^3\iota}{N^k(x_t, a_t)}} +  \frac{8e^2 H^5S^2\iota^2}{\bracket{N^k(x_t, a_t) }^2} },
\end{align*}
inserting which to \eqref{eq:Ek-first-bound} we have that 
\[
    E^k_h(x,a) 
     \le \sqrt{\frac{2 H^2\iota}{N^k(x,a)}}  + \frac{ H S\iota}{N^k(x,a)} 
    +  \Ebb_{\pi^k, \Pbb} \sum_{t\ge h+1} \bracket{ {\frac{4e^2 H^3\iota}{N^k(x_t, a_t)}} +  \frac{8e^2 H^5S^2\iota^2}{\bracket{N^k(x_t, a_t) }^2} },
\]
where $(x_h, a_h) = (x,a)$.
The two upper bounds on $E^k_h(x,a)$ together complete the proof.
\end{proof}

\begin{lemma}[Bounds for the lower order term]\label{lemma:lower-order-term}
    If $G_2$ and $G_3$ hold, we have that for every $V_1, V_2$ such that $0 \le V_1(x) \le V_2(x) \le M$ and for every $k,x,a$, the following inequalities hold:
    \begin{align*}
        &  \abs{(\widehat{\Pbb}^k - \Pbb) (V_2 - V_1) (x,a) }
        \le \Pbb\bracket{V_2 - V_1}^2 (x,a)  + \frac{ M S\iota}{N^k(x,a)}; \\
        & \abs{(\widehat{\Pbb}^k - \Pbb) (V_2 - V_1) (x,a) }
        \le \frac{1}{H} \Pbb\bracket{V_2 - V_1} (x,a)  + \frac{2 M H S\iota}{N^k(x,a)};\\
        & \abs{ (\widehat{\Pbb}^k - \Pbb) (V_2 - V_1) (x,a) }
        \le \frac{1}{H} \widehat{\Pbb}^k\bracket{V_2 - V_1} (x,a)  + \frac{ 3 M H S\iota}{N^k(x,a)}.
    \end{align*}
\end{lemma}
\begin{proof}
    For simplicity let us denote $p(y) := \Pbb(y \mid x,a)$ and $\hat{p}(y) := \widehat{\Pbb}^k(y \mid x,a)$.
    For the first inequality,
    \begingroup
\allowdisplaybreaks
\begin{align*}
    &\ \abs{(\widehat{\Pbb}^k - \Pbb) (V_2 - V_1) (x,a)} 
     \le \sum_{y\in \Scal} \abs{\hat{p}^k(y) - p(y)}\bracket{V_2(y) - V_1(y)} \\
    &\le \sum_{y\in \Scal} \left(\sqrt{\frac{2p(y)\iota}{N^k(x,a)}} +\frac{2\iota}{3N^k(x,a)} \right)\bracket{V_2(y) - V_1(y)} \qquad (\text{since $G_2$ holds})\\
    &\le \sum_{y\in \Scal} \sqrt{\frac{2\iota}{N^k(x,a)}} \cdot \sqrt{ p(y) \bracket{V_2(y) - V_1(y)}^2 } +\frac{2MS\iota}{3N^k(x,a)} \\
    &\le \sum_{y\in \Scal} \bracket{ \frac{\iota}{N^k(x,a)} + p(y)\bracket{V_2(y) - V_1(y)}^2  }   +  \frac{2MS\iota}{3N^k(x,a)}  \qquad(\text{use $\sqrt{ab} \le a + b$})\\
    &\le \Pbb\bracket{V_2 - V_1}^2 (x,a)  + \frac{ M S\iota}{N^k(x,a)}.
\end{align*}
\endgroup
For the second inequality,
\begin{align*}
    &\ \abs{ (\widehat{\Pbb}^k - \Pbb) (V_2 - V_1) (x,a) } 
     \le \sum_{y\in \Scal} \abs{\hat{p}^k(y) - p(y)}\bracket{V_2(y) - V_1(y)} \\
    &\le \sum_{y\in \Scal} \left(\sqrt{\frac{2p(y)\iota}{N^k(x,a)}} +\frac{2\iota}{3N^k(x,a)} \right)\bracket{V_2(y) - V_1(y)}  \qquad (\text{since $G_2$ holds}) \\
    &\le \sum_{y\in \Scal} \left(\frac{p(y)}{H} + \frac{H\iota}{2N^k(x,a)} +\frac{2\iota}{3N^k(x,a)} \right)\bracket{V_2(y) - V_1(y)} \qquad(\text{use $\sqrt{ab} \le \half (a + b)$}) \\
    &\le \sum_{y\in \Scal} \frac{p(y)}{H}\bracket{V_2(y) - V_1(y)} +  \frac{2 M H S\iota}{N^k(x,a)}  \\
    &= \frac{1}{H}\Pbb\bracket{V_2 - V_1} (x, a)  + \frac{ 2 M H S\iota}{N^k(x,a)}.
\end{align*}
The third inequality is proved in a same way as the second inequality, except that in the second step we use event $G_3$ rather than $G_2$.
\end{proof}

\begin{lemma}[Half-clip trick]\label{lemma:half-clip}
    If $G_1$ holds, then for every $k$,
    \[V^*_1(x_1) - V_1^{\pi^k}(x_1) \le 2 \cdot \Ebb_{\pi^k, \Pbb} \sum_{h=1}^H \clip_{\frac{\rho}{2H}} [E^k_h(x_h, a_h)].\]
\end{lemma}
\begin{proof}
    Following \citep{simchowitz2019non}, 
    let us consider 
    \begin{equation}\label{eq:clipped-surplus}
        \ddot{E}^k_h(x,a) := \clip_{\frac{\rho}{2H}} [E^k_h(x, a)] = E^k_h(x, a)\cdot \ind{E^k_h(x, a) \ge \frac{\rho}{2H}}  \ge 0 
    \end{equation}
    and
    \begin{equation*}
        \begin{cases}
        \ddot{V}^{k}_h(x) := \ddot{Q}^k_h(x, \pi^k(x)), \\
        \ddot{Q}^k_h(x,a) := r_h(x,a) + \ddot{E}^k_h(x,a) + \Pbb \ddot{V}^{k}_{h+1}(x,a).
        \end{cases}
    \end{equation*}
    Notice that 
    \begin{gather}
        V^k_h(x) - V^{\pi^k}_h(x) = \Ebb_{\pi^k. \Pbb} \sum_{t\ge h}  {E}^k_t(x_t,a_t), \label{eq:Vk-Vpi-decomp}\\
        \ddot{V}^k_h(x) - V^{\pi^k}_h(x) = \Ebb_{\pi^k, \Pbb} \sum_{t\ge h} \ddot{E}^k_t(x_t,a_t) \ge 0 , \label{eq:Vkddot-Vpi-decomp}
    \end{gather}
    then we immediately see that
    \begin{align}
        \ddot{V}^k_h(x) - V^{\pi^k}_h(x) 
        &\ge \Ebb_{\pi^k. \Pbb} \sum_{t\ge h} \bracket{ {E}^k_t(x_t,a_t) - \frac{\rho}{2H}} 
        = V^k_h(x) - V^{\pi^k}_h(x) - \frac{H-h +1}{H} \cdot \frac{\rho}{2} \notag \\
        &\ge V^k_h(x) - V^{\pi^k}_h(x) - \frac{\rho}{2}. \label{eq:Vkddot-Vpi-lower-bound}
    \end{align}
    Given a sequence of a random trajectory $\set{x_h, a_h}_{h=1}^H$ induced by policy $\pi^k$ and $\Pbb$, let $F_h$ be the event such that 
    \[F_h := \set{ a_h \notin \pi^*_h(x_h),\ \forall t < h,\ a_t \in \pi^*_t(x_t) }.\]
    Clearly $\set{F_h}_{h=1}^{H+1}$ are disjoint and form a partition for the sample space of the random trajectory induced by policy $\pi^k$.
    Then we have that 
    \begin{align*}
        V^*_1(x_1) - V^{\pi^k}_1(x_1)
        &= \Ebb_{\pi^k, \Pbb} \sum_{h=1}^H \ind{F_h} \bracket{ V^*_1(x_1) - V^{\pi^k}_1(x_1) }  + 0 \\
        &= \Ebb_{\pi^k, \Pbb} \sum_{h=1}^H \ind{F_h} \bracket{ V^*_h(x_h) - V^{\pi^k}_h(x_h) } \\
        &= \Ebb_{\pi^k, \Pbb} \sum_{h=1}^H \ind{F_h} \bracket{ \gap_h(x_h, a_h) + Q^*_h(x_h, a_h) - V^{\pi^k}_h(x_h) },
    \end{align*}
    where $\gap_h(x_h, a_h) \ge \rho > 0$ under $F_h$ (since $a_h \notin \pi^*_h(x_h)$).
    Similarly we have that 
    \begingroup
\allowdisplaybreaks
    \begin{align*}
        \ddot{V}^k_1(x_1) - V^{\pi^k}_1(x_1)
        &= \Ebb_{\pi^k, \Pbb} \sum_{h=1}^H \ind{F_h} \bracket{ \ddot{V}^k_1(x_1) - V^{\pi^k}_1(x_1) } + \ind{F_{H+1}} \bracket{ \ddot{V}^k_1(x_1) - V^{\pi^k}_1(x_1) }\\
        &\ge \Ebb_{\pi^k, \Pbb} \sum_{h=1}^H \ind{F_h} \bracket{ \ddot{V}^k_h(x_h) - V^{\pi^k}_h(x_h) } \qquad (\text{use \eqref{eq:Vkddot-Vpi-decomp} and \eqref{eq:clipped-surplus}})\\
        &\ge \Ebb_{\pi^k, \Pbb} \sum_{h=1}^H \ind{F_h} \bracket{ {V}^k_h(x_h) - V^{\pi^k}_h(x_h) - \frac{\rho}{2} } \qquad (\text{use \eqref{eq:Vkddot-Vpi-lower-bound}}) \\
        &\ge \Ebb_{\pi^k, \Pbb} \sum_{h=1}^H \ind{F_h} \bracket{ {V}^*_h(x_h) - V^{\pi^k}_h(x_h) - \frac{1}{2} \gap_{h}(x_h,a_h)} \\
        &\qquad (\text{use Lemma \ref{lemma:optimistic-planning}, and that $\gap_{h}(x_h,a_h) \ge \rho$ under $F_h$})\\
        &= \Ebb_{\pi^k, \Pbb} \sum_{h=1}^H \ind{F_h} \bracket{  \half \gap_{h}(x_h,a_h)+ {Q}^*_h(x_h,a_h) - V^{\pi^k}_h(x_h) } \\
        &\ge \frac{1}{2}  \Ebb_{\pi^k, \Pbb} \sum_{h=1}^H \ind{F_h} \bracket{ \gap_{h}(x_h,a_h)+ {Q}^*_h(x_h,a_h) - V^{\pi^k}_h(x_h) } \\
        &\qquad (\text{note that ${Q}^*_h(x_h,a_h) - V^{\pi^k}_h(x_h) \ge 0$}) \\
        &= \frac{1}{2} \bracket{V^*_1(x_1) - V^{\pi^k}_1(x_1) }.
    \end{align*}
    \endgroup
    The above inequality plus \eqref{eq:Vkddot-Vpi-decomp} \eqref{eq:clipped-surplus} completes the proof.
\end{proof}
\begin{lemma}\label{lemma:planning-error-bounded-by-sum-of-bonus}
    If $G_1$, $G_2$ and $G_3$ hold, then
    \begin{align*}
    &V^*_1 (x_1) - V_1^{\pi^k} (x_1)
    \le \\
    &\qquad\qquad \Ebb_{\pi^k, \Pbb} \sum_{h=1}^H \Biggl( \clip_{\frac{\rho}{2H}} \sbracket{ \sqrt{\frac{8 H^2\iota}{N^k(x_h,a_h)}} } + \frac{120(H^3 + S)H \iota}{N^k(x_h,a_h)} +  \frac{240 H^6 S^2\iota^2}{\bracket{N^k(x_h, a_h) }^2} \Biggr).
\end{align*}
\end{lemma}
\begin{proof} We proceed the proof as follows:
\begingroup
\allowdisplaybreaks
    \begin{align*}
        & \ V^*_1(x_1) - V_1^{\pi^k} (x_1)
        \le 2\cdot \Ebb_{\pi^k, \Pbb} \sum_{h=1}^H \clip_{\frac{\rho}{2H}} [E^k_h(x_h, a_h)] \qquad (\text{use Lemma \ref{lemma:half-clip}})\\
        &\le  2\cdot \Ebb_{\pi^k, \Pbb} \sum_{h=1}^H  \clip_{\frac{\rho}{2H}} \Bigg[ \sqrt{\frac{2 H^2\iota}{N^k(x,a)}}  + \frac{HS\iota}{N^k(x,a)} \\
        &\qquad \qquad \qquad + \Ebb_{\pi^k, \Pbb} \sum_{t\ge h+1} \bracket{ {\frac{4 e^2 H^3\iota}{N^k(x_t, a_t)}} +  \frac{8 e^2 H^5S^2\iota^2}{\bracket{N^k(x_t, a_t) }^2} } \Bigg]  \qquad (\text{use Lemma \ref{lemma:surplus-bound}}) \\
        &\le 2\cdot \Ebb_{\pi^k, \Pbb} \sum_{h=1}^H \Bigg\{ \clip_{\frac{\rho}{4H}} \sbracket{ \sqrt{\frac{2 H^2\iota}{N^k(x_h,a_h)}} } + \frac{2 HS\iota}{N^k(x_h,a_h)} \\
        &\qquad\qquad\qquad + \Ebb_{\pi^k, \Pbb} \sum_{t\ge h+1} \bracket{ {\frac{8 e^2 H^3\iota}{N^k(x_t, a_t)}} +  \frac{16 e^2 H^5S^2\iota^2}{\bracket{N^k(x_t, a_t) }^2} }  \Bigg\} \qquad (\text{use Lemma \ref{lemma:clip-operator}})\\
        &\le 2\cdot \Ebb_{\pi^k, \Pbb} \sum_{h=1}^H \set{ \clip_{\frac{\rho}{4H}} \sbracket{ \sqrt{\frac{2 H^2\iota}{N^k(x_h,a_h)}} } + \frac{2HS\iota}{N^k(x_h,a_h)} } \\
        & \qquad \qquad \qquad + 2 H\cdot \Ebb_{\pi^k, \Pbb} \sum_{t=1}^H \bracket{ {\frac{8 e^2 H^3\iota}{N^k(x_t, a_t)}} +  \frac{16 e^2 H^5S^2\iota^2}{\bracket{N^k(x_t, a_t) }^2} }  \\
        &= \Ebb_{\pi^k, \Pbb} \sum_{h=1}^H \bracket{ \clip_{\frac{\rho}{2H}} \sbracket{ \sqrt{\frac{8H^2\iota}{N^k(x_h,a_h)}} } + \frac{4HS\iota}{N^k(x_h,a_h)} + {\frac{16 e^2 H^4\iota}{N^k(x_h, a_h)}} +  \frac{32 e^2 H^6 S^2\iota^2}{\bracket{N^k(x_h, a_h) }^2}  } \\
        &\le \Ebb_{\pi^k, \Pbb} \sum_{h=1}^H \bracket{ \clip_{\frac{\rho}{2H}} \sbracket{ \sqrt{\frac{8H^2\iota}{N^k(x,a)}} } +  \frac{120 (H^3 +S)H \iota }{N^k(x,a)}  +  \frac{240 H^6 S^2\iota^2}{\bracket{N^k(x, a) }^2} }.
    \end{align*}
    \endgroup
\end{proof}

\paragraph{Exploration Phase.}
Recall the exploration bonus in Algorithm \ref{alg:exploration} is defined as
\begin{equation}\label{eq:exploration-bonus}
    c^k(x,a) :=   \clip_{\frac{\rho}{2H}} \sbracket{ \sqrt{\frac{8H^2\iota}{N^k(x,a)}} } +  \frac{120 (H+S)H^3 \iota}{N^k(x,a)}  +  \frac{240 H^6 S^2\iota^2}{\bracket{N^k(x, a) }^2},
\end{equation}
and the exploration value function in Algorithm \ref{alg:exploration} is given by 
\begin{equation}\label{eq:exploration-sum-of-bonus}
    \begin{cases}
        \overline{V}_h^{k}(x) = \max_a \overline{Q}_h^{k}(x, a), \\
        \overline{Q}_h^{k}(x,a) = H\land \bracket{ c^k(x,a) + \widehat{\Pbb}^k \overline{V}_{h+1}^{k}(x,a) }.
    \end{cases}
\end{equation}
Let us also define the following population and empirical \emph{bonus value functions} (for some policy $\pi$):
\begin{align}
    \begin{cases}
        \widetilde{V}_h^{k, \pi}(x) = \widetilde{Q}_h^{k, \pi}(x,\pi(x)), \\
        \widetilde{Q}_h^{k, \pi}(x,a) = H \land \bracket{ c^k(x,a) + \Pbb \widetilde{V}_{h+1}^{k, \pi}(x,a) };
    \end{cases} \label{eq:population-sum-of-bonus}\\
    \begin{cases}
        \overline{V}_h^{k, \pi}(x) = \overline{Q}_h^{k, \pi}(x,\pi(x)), \\
        \overline{Q}_h^{k, \pi}(x,a) =  H \land \bracket{c^k(x,a) + \widehat{\Pbb}^k \overline{V}_{h+1}^{k, \pi}(x,a) }. \label{eq:empirical-sum-of-bonus}
    \end{cases}
\end{align}

\begin{lemma}[Exploration value function maximizes empirical bonus value functions]\label{lemma:bonus-value-bounded-by-exploration-value}
    For every $\pi$ and every $k,h,x,a$, \[ \overline{Q}_h^{k, \pi}(x,a) \le \overline{Q}_h^{k}(x,a)\quad \text{and} \quad \overline{V}_h^{k, \pi}(x) \le \overline{V}_h^{k}(x). \]
\end{lemma}
\begin{proof}
    Use induction and \eqref{eq:exploration-sum-of-bonus} \eqref{eq:empirical-sum-of-bonus}.
\end{proof}

\begin{lemma}[Planning error is upper bounded by population bonus value function]\label{lemma:planning-error-bounded-by-bonus-value}
    For every $\pi^k$
    \begin{equation}
        V^*_1(x_1) - V_1^{\pi^k}(x_1) \le  \widetilde{V}_1^{k, \pi^k}(x_1). \label{eq:planning-error-bounded-by-sum-of-bonus}
        \end{equation}
\end{lemma}
\begin{proof}
Let $\Acal_h$ be the $\sigma$-field generated by $\set{x_1, a_1, \dots, x_h, a_h}$ (induced by $\pi^k$ and $\Pbb$).
For simplicity, denote $\Ebb_{\ge h} [\cdot]:= \Ebb [ \cdot \mid \Acal_{h-1} ]$, i.e., taking conditional expectation given a trajectory $\set{x_1, a_1, \dots, x_{h-1}, a_{h-1}}$. Then $\Ebb_{\ge 1} [\cdot]$ is taking the full expectation.
From \eqref{eq:population-sum-of-bonus} we obtain
\begingroup
\allowdisplaybreaks
\begin{align*}
    &\ \Ebb_{\ge h}  \sbracket{  \widetilde{Q}^{k, \pi^k}_h(x_h, a_h)  }
    = \Ebb_{\ge h} \sbracket{ H \land \bracket{ c(x_h, a_h) + \Ebb_{ \ge h+1} \sbracket{ \widetilde{V}^{k, \pi^k}_{h+1}(x_{h+1})} } }\\
    &= \Ebb_{\ge h} \sbracket{ H \land \bracket{ c(x_h, a_h) + \Ebb_{ \ge h+1} \sbracket{ \widetilde{Q}^{k, \pi^k}_{h+1}(x_{h+1}, a_{h+1})} } }\\
    &= \Ebb_{\ge h} \sbracket{ H \land \Ebb_{ \ge h+1} \sbracket{ c(x_h, a_h) + \widetilde{Q}^{k, \pi^k}_{h+1} (x_{h+1}, a_{h+1}) } } \\
    &\ge \Ebb_{\ge h} \Ebb_{ t \ge h+1} \sbracket{ H \land \bracket{ c(x_h, a_h) + \widetilde{Q}^{k, \pi^k}_{h+1} (x_{h+1}, a_{h+1}) } } \qquad(\text{$H\land \cdot$ is concave}) \\
    &= \Ebb_{\ge h} \sbracket{ H \land \bracket{ c(x_h, a_h) + \widetilde{Q}^{k, \pi^k}_{h+1} (x_{h+1}, a_{h+1}) } }.
\end{align*}
\endgroup
Recursively applying the above relation, and using a fact that 
\[ H \land \bracket{ a + H\land b} = H \land \bracket{ a + b} \text{ for } a, b \ge 0, \]
we obtain that 
\begin{align*}
    \widetilde{V}^{k, \pi^k}_1(x_1)
    = \Ebb_{\ge 1}  \sbracket{  \widetilde{Q}^{k, \pi^k}_1(x_1, a_1)  }
     \ge \Ebb_{\ge 1} \sbracket{ H \land \sum_{h=1}^H c(x_h, a_h)} 
    = \Ebb_{\pi^k, \Pbb} \sbracket{ H \land \sum_{h=1}^H c(x_h, a_h)}.
\end{align*}
Finally, by Lemma \ref{lemma:planning-error-bounded-by-sum-of-bonus} and that $V^*_1(x_1) - V_1^{\pi^k}(x_1) \le H = \Ebb_{\pi^k, \Pbb}\sbracket{ H }$, we have that 
\begin{align*}
    V^*_1(x_1) - V_1^{\pi^k}(x_1) 
    &\le \Ebb_{\pi^k, \Pbb}\sbracket{ H \land \sum_{h=1}^H c(x_h, a_h)} \le \widetilde{V}^{k, \pi^k}_1(x_1),
\end{align*}
which completes our proof.
\end{proof}

    


\begin{lemma}[Empirical vs. population bonus value functions]\label{lemma:population-and-empirical-bonus-value}
    If $G_2$ and $G_3$ hold, then for every $k$ and for every policy $\pi$, we have that
    \[
        \frac{1}{e} \cdot \overline{V}_1^{k, \pi}(x_1) \le  \widetilde{V}_1^{k, \pi}(x_1) \le e \cdot \overline{V}_1^{k, \pi}(x_1). 
    \]
\end{lemma}
\begin{proof}
    For the second inequality, we only need to prove that for every $k$ and $\pi$,
    \[
        \widetilde{V}_h^{k, \pi}(x) \le \bracket{1+ \frac{1}{H}}^{H-h+1} \cdot \overline{V}_h^{k, \pi}(x),  \text{ for every } h, x.
    \]
    We proceed by induction over $h$.
    For $H+1$ the hypothesis holds trivially as $\widetilde{V}_{H+1}^{k, \pi}(x) = 0 =\overline{V}_{H+1}^{k, \pi}(x) $. Now suppose the hypothesis holds for $h+1$, and let us consider $h$:
    \begin{align*}
        \widetilde{V}_h^{k, \pi}(x) &= \widetilde{Q}_h^{k, \pi} (x,a) \qquad(\text{set $a = \pi(x)$}) \\
        &\le c^k(x,a) + \Pbb \widetilde{V}_{h+1}^{\pi} (x,a) \qquad(\text{by \eqref{eq:population-sum-of-bonus}}) \\
        &= c^k(x,a) + \widehat{\Pbb}^k \widetilde{V}_{h+1}^{k, \pi} (x,a) + (\Pbb - \widehat{\Pbb}^k)\widetilde{V}_{h+1}^{k, \pi} (x,a) \\
        &\le c^k(x,a) + \bracket{1+\frac{1}{H}}\widehat{\Pbb}^k \widetilde{V}_{h+1}^{k, \pi}(x,a) + \frac{3H^2 S\iota}{N^k(x,a)} \quad(\text{by Lemma \ref{lemma:lower-order-term} and $\widetilde{V}_{h+1}^{k, \pi} \le H$}) \\
        &\le \bracket{1+\frac{1}{H}} c^k(x,a) + \bracket{1+\frac{1}{H}}\widehat{\Pbb}^k \widetilde{V}_{h+1}^{k,\pi} (x,a). \qquad(\text{by \eqref{eq:exploration-bonus}}) \\
        &\le \bracket{1+\frac{1}{H}} c^k(x,a) + \bracket{1+\frac{1}{H}}^{H-h+1}\widehat{\Pbb}^k \overline{V}_{h+1}^{k,\pi} (x,a) \qquad(\text{by induction hypothesis}) \\
        &\le \bracket{1+\frac{1}{H}}^{H-h+1} \cdot \bracket{ c^k(x,a) + \widehat{\Pbb}^k \overline{V}_{h+1}^{k,\pi} (x,a)},
    \end{align*}
    moreover from \eqref{eq:population-sum-of-bonus} we have that $\widetilde{V}_h^{k, \pi}(x) \le H$. In sum, we have   
    \begin{align*}
        \widetilde{V}_h^{k, \pi}(x) 
        &\le \bracket{1+ \frac{1}{H}}^{H-h+1} \cdot \bracket{ H \land \bracket{c^k(x,a) + \widehat{\Pbb}^k \overline{V}_{h+1}^{k,\pi} (x,a)}} \\
        &= \bracket{1+ \frac{1}{H}}^{H-h+1} \cdot \overline{Q}_{h+1}^{k,\pi} (x,a)
        = \bracket{1+ \frac{1}{H}}^{H-h+1} \cdot \overline{V}_{h+1}^{k,\pi} (x).
    \end{align*}
This completes our induction, and as a consequence proves the second inequality.

The first inequality is proved by repeating the above argument to show that for every $k$ and $\pi$
\[ \overline{V}_h^{k, \pi}(x) \le \bracket{1+ \frac{1}{H}}^{H-h+1} \cdot \widetilde{V}_h^{k, \pi}(x), \text{ for every } h, x. \]
\end{proof}

\begin{lemma}[Exploration regret]\label{lemma:exploration-regret}
    If $G_2$, $G_3$ and $G_4$ hold, then
    \[
    \sum_{k=1}^K  \overline{V}_1^{k}(x_1) \lesssim \frac{H^3SA\iota}{\rho} +  H^4 S^2 A\iota \cdot \log(HK).
    \]
\end{lemma}
\begin{proof}
    Recall $\bar{\pi}^k$ is the exploration policy at the $k$-th episode, i.e., a greedy policy given by maximizing $\overline{Q}^k_h(x,a)$. 
    Recall $w^k_h(x,a) := \Pbb \set{(x_h, a_h) = (x,a) \mid \bar{\pi}^k, \Pbb}$ and $w^k(x,a) = \sum_{h} w^k_h(x,a) $.
Let us consider the following ``good sets'':
\begin{equation}\label{eq:good-sets}
L^k := \big\{(x,a) : \sum_{j < k}w^j(x,a) \ge H^3 S \iota \big\}.
\end{equation}
Then we have
\begin{align*}
  &\  \sum_{k=1}^K  \overline{V}_1^{k}(x_1)
  \le e \cdot \sum_{k=1}^K  \widetilde{V}_1^{k, \bar{\pi}^k} (x_1) \qquad(\text{use Lemma \ref{lemma:population-and-empirical-bonus-value}}) \\
&\le e\cdot \sum_{k=1}^K \sum_{x,a \in L^k} \sum_{h=1}^H  w^k_h(x,a) c^k(x,a) 
    + e \cdot \sum_{k=1}^K \sum_{x,a \notin L^k} \sum_{h=1}^H  w^k_h(x,a) H  \qquad (\text{use \eqref{eq:population-sum-of-bonus}}) \\
&= e\cdot \sum_{k=1}^K \sum_{x,a \in L^k}  w^k(x,a) c^k(x,a) + e\cdot \sum_{k=1}^K \sum_{x,a \notin L^k}  w^k(x,a) H \\
&\le e\cdot \sum_{k=1}^K \sum_{x,a \in L^k} w^k(x,a) \bracket{ \clip_{\frac{\rho}{2H}} \sbracket{ \sqrt{\frac{8H^2\iota}{N^k(x,a)}} } +  \frac{120(H+S)H^3\iota }{N^k(x,a)}  +   \frac{240 H^6 S^2\iota^2}{\bracket{N^k(x, a) }^2} }\\
&\qquad + e\cdot \sum_{k=1}^K \sum_{x,a \notin L^k} w^k(x,a) H \qquad (\text{use \eqref{eq:exploration-bonus}}).
\end{align*} 
We then bound each terms using integration tricks and that $G_4$ holds.
    The fourth term is bounded by 
    \begin{align*}
        e\sum_{k=1}^K \sum_{x,a \notin L^k} w^k(x,a) H \lesssim S A \cdot  (H + H^3 S \iota) \cdot H \lesssim H^4 S^2 A \iota.
    \end{align*}
    The third term is bounded by 
    \begin{align*}
        &\ e\sum_{k=1}^K \sum_{x,a \in L^k} w^k(x,a) \frac{240 H^6 S^2\iota^2}{\bracket{N^k(x, a) }^2} \\
        &\lesssim H^6 S^2\iota^2\sum_{k=1}^K \sum_{x,a} \frac{w^k(x,a)}{\bracket{\sum_{j < k} w^j(x,a) - 2H\iota}^2}\cdot \ind{\sum_{j< k} w^j(x,a) \ge H^3 S \iota} \ (\text{use \ref{eq:G4} and \eqref{eq:good-sets}})\\
        &\lesssim H^6 S^2\iota^2\cdot SA \cdot \frac{1}{H^3 S \iota - 2H\iota} 
        \lesssim H^3 S^2 A\iota. \qquad (\text{integration trick})
    \end{align*}
    The second term is bounded by 
    \begin{align*}
        &\ e\sum_{k=1}^K \sum_{x,a \in L^k} w^k(x,a) \frac{120(H+ S)H^3\iota}{N^k(x,a)} \\
        &\lesssim (H + S)H^3\iota \sum_{k=1}^K \sum_{x,a} \frac{w^k(x,a)}{\sum_{j< k} w^j(x,a) - 2H\iota}\cdot \ind{\sum_{j< k} w^j(x,a) \ge  H^3 S \iota} \ (\text{use \ref{eq:G4} and \eqref{eq:good-sets}}) \\
        &\lesssim (H+S)H^3\iota \cdot SA \cdot \log (HK) 
        \lesssim (H+S)H^3 SA\iota \cdot \log (HK) .  \qquad (\text{integration trick})
    \end{align*}
    The first term is bounded by 
    \begingroup
\allowdisplaybreaks
    \begin{align*}
        &\quad e\sum_{k=1}^K \sum_{x,a \in L^k} w^k(x,a)\clip_{\frac{\rho}{2H}} \sbracket{ \sqrt{\frac{8H^2\iota}{N^k(x,a)}} } \\
        &= e\sum_{k=1}^K \sum_{x,a} w^k(x,a) \cdot \sqrt{\frac{8H^2\iota}{N^k(x,a)}}  \cdot \ind{\sum_{j< k} w^j(x,a) \ge H^3 S \iota} \cdot \ind{N^k(x,a) \le \frac{32H^4\iota }{\rho^2}} \\
        &\lesssim H\sqrt{\iota} \cdot \sum_{k=1}^K \sum_{x,a} \frac{w^k(x,a)}{\sqrt{\sum_{j< k} w^j(x,a) - 2H\iota}}\cdot \ind{ \sum_{j< k} w^j(x,a) \le \frac{64H^4\iota }{\rho^2} + 2H\iota } \ (\text{use \ref{eq:G4}})\\
        &\lesssim H\sqrt{\iota} \cdot SA \cdot \sqrt{\frac{H^4\iota}{\rho^2} + H\iota} 
        \lesssim \frac{H^3 SA \iota}{\rho}.  \qquad (\text{integration trick})
    \end{align*}
    \endgroup
    Summing up everything yields that 
    \[
        \sum_{k=1}^K  \overline{V}_1^{k}(x_1)
        \lesssim     \frac{H^3 SA \iota}{\rho} + (H+S)H^3 SA\iota\cdot \log(HK) + H^3 S^2 A \iota + H^4 S^2 A \iota
        \lesssim  \frac{H^3 SA \iota}{\rho} + H^4 S^2 A \iota \cdot \log(H K).
    \]
\end{proof}

\begin{theorem}[Restatement of Theorem \ref{thm:gap-upper-bound}]\label{thm:gap-upper-bound-restate}
    With probability at least $1-\delta$, the planning error is bounded by
\[
    V^*_1(x_1) - V_1^{\pi}(x_1) \lesssim \frac{H^3SA}{\rho K}\cdot \log \frac{HSAK}{\delta} + \frac{H^4 S^2 A }{K}\cdot \log(HK) \cdot \log \frac{HSAK}{\delta}.
\]
\end{theorem}
\begin{proof}
    First by Lemma \ref{lemma:good-probability}, we have that with probability at least $1-\delta$, $G_1$, $G_2$, $G_3$ and $G_4$ hold. Next we have the following:
    \begin{align*}
        V^*_1(x_1) - V_1^{\pi}(x_1) &= \frac{1}{K} \sum_{k=1}^K \bracket{V^*_1(x_1) - V_1^{\pi^k}(x_1) } \qquad  (\text{by Algorithm \ref{alg:planning}})\\
        &\le \frac{1}{K} \sum_{k=1}^K \widetilde{V}_1^{k, \pi^k}(x_1) \qquad  (\text{by Lemma \ref{lemma:planning-error-bounded-by-bonus-value}})\\ 
        &\le \frac{e}{K} \sum_{k=1}^K \overline{V}_1^{k, \pi^k} (x_1) \qquad  (\text{by Lemma \ref{lemma:population-and-empirical-bonus-value}})\\
        &\le \frac{e}{K} \sum_{k=1}^K \overline{V}_1^{k}(x_1) \qquad  (\text{by Lemma \ref{lemma:bonus-value-bounded-by-exploration-value}}) \\
        &\lesssim \frac{H^3SA\iota}{\rho K} + \frac{H^4 S^2 A \iota \log(HK)}{K}. \qquad(\text{by Lemma \ref{lemma:exploration-regret}})
    \end{align*}
\end{proof}

\begin{lemma}[Properties of the clip operator]\label{lemma:clip-operator}
    Let $\rho, \rho', a > 0$, then 
    \begin{itemize}
        \item \( a\cdot \clip_\rho\sbracket{A} = \clip_{a\rho}\sbracket{a\cdot A}; \)
        \item Let $\rho \ge \rho'$ and $A \le A'$, then \(A- \rho \le \clip_{\rho}[A] \le \clip_{\rho'}[A'] \le A'\);
        \item \(\clip_\rho\sbracket{A+B} \le \clip_{\frac{\rho}{2}}\sbracket{A} + 2 B \) for $B \ge 0$;
        \item \(\clip_{\rho}[A_1 + \cdots + A_m] \le 2 \cbracket{\clip_{\frac{\rho}{2m}}[A_1] + \cdots + \clip_{\frac{\rho}{2m}}[A_m] }.\)
    \end{itemize}
\end{lemma}
\begin{proof}
The first three claims are easy to see by the definition of the clip operator.
The last claim is from \citep{simchowitz2019non}, for which we provided a proof here for completeness.
Without loss of generality, assume that $A_1 + \cdots + A_m \ge \rho$. Let us divide $\set{A_i}_{i=1}^m$ into two groups by examining whether or not $A_i \ge \frac{\rho}{2m}$. Without loss of generality, assume that 
\[
A_1, \dots, A_k \ge \frac{\rho}{2m},\qquad 
A_{k+1}, \dots, A_m < \frac{\rho}{2m}.
\]
The latter implies that \( A_{k+1} + \dots + A_m < \frac{\rho}{2m} \cdot (m-k) \le \frac{\rho}{2} \), then by \( A_1 + \dots + A_m \ge \rho\) we obtain that   
\[
A_1 + \dots + A_k \ge \frac{\rho}{2} > A_{k+1} + \dots + A_k,
\]
In sum, we have that 
\begin{align*}
    \text{RHS} 
    &= 2\cbracket{\clip_{\frac{\rho}{2m}} [ A_1 ] + \dots + \clip_{\frac{\rho}{2m}} [A_m]} \\
    &= 2\cbracket{A_1 + \dots + A_k}\\
    &\ge A_1 + \dots + A_k + A_{k+1} + \dots + A_m \\
    &= \text{LHS}.
\end{align*}
\end{proof}

\section{PROOF OF THE LOWER BOUND (THEOREM \ref{thm:gap-lower-bound})}\label{appendix-section:lower-bound-proof}

\begin{lemma}[\citep{mannor2004sample}, Theorem 1]\label{lemma:bandit-lb}
    There exist positive constants $c_1, c_2, \epsilon_0, $ and $\delta_0$, such that for every $ n \ge 2$, $\epsilon \in (0, \epsilon_0)$, $\delta \in (0, \delta_0)$, and for every $(\epsilon, \delta)$-correct policy, there exists some Bernoulli multi-armed bandit model with $n$ arms, such that the policy needs at least $T$ number of trials where
    \[
    \Ebb [T] \ge c_1 \frac{n}{\epsilon^2} \log \frac{c_2}{\delta}.
    \]
    In particular, the bandit model can be chosen such that one arm pays $1$ w.p. $1/2 + \epsilon/2$, one arm pays $1$ w.p. $1/2 + \epsilon$, and the rest arms pay $1$ w.p. $1/2$.
\end{lemma}

\begin{theorem}[Restatement of Theorem \ref{thm:gap-lower-bound}]\label{thm:gap-lower-bound-restate}
    Fix $S\ge 5, A\ge 2, H \ge 2+\log_A S$. 
    There exist positive constants $c_1, c_2, \rho_0, \delta_0$, such that for every $\rho \in (0, \rho_0)$, $\epsilon \in (0, \rho)$, $\delta\in (0,\delta_0)$, and for every $(\epsilon, \delta)$-correct policy, there exists some MDP instance with gap $\rho$, such that 
    \[
    \Ebb [K] \ge c_1 \frac{H^2 SA}{\epsilon \rho}\log\frac{c_2}{\delta}.    
    \]
\end{theorem}
\begin{proof}
The hard example is constructed as in Figure \ref{fig:hard-example}. We prove such example witness our lower bound as follows.

Let us call all left orange states the bandit states. 
Let $N_b$ be the number of visits to the bandit states. Then from the construction, we have that 
\[\Ebb[N_b] = \Ebb[K] \cdot \frac{\epsilon}{\rho}.\]
We may without loss of generality image the bandit states as an entity, and at this entity, there are $SA$ arms: one arm pays reward $H$ w.p. $\frac{1}{2} + \frac{\rho}{2H}$, one arm pays reward $H$ w.p. $\frac{1}{2} + \frac{\rho}{H}$, and the rest arms pay reward $H$ w.p. $\frac{1}{2}$.
Next, for any $(\epsilon, \delta)$-correct policy on the MDP, it induces a policy that is $(\rho, \delta)$-correct policy on the above bandit model. By a linear scaling of the reward from $H$ to $1$, it is equivalent to a policy that is $(\frac{\rho}{H}, \delta)$-correct in a stand hard-to-learn bandit model with $SA$-arms.

By Lemma \ref{lemma:bandit-lb}, we must have that 
\[\Ebb [N_b] \ge c_1 \frac{H^2SA}{\epsilon^2}\log \frac{c_2}{\delta}, \]
which implies that 
\[\Ebb[K] \ge c_1 \frac{H^2SA}{\epsilon\rho}\log \frac{c_2}{\delta}.\]

Clearly, the MDP discussed above has $A$ actions, $2S$ states, $H+2+\log_A S$ length of the horizon. We next verify that the MDP has $\frac{\rho}{2}$ gap.
Notice that except in the left orange states, all actions have the same consequence, therefore the gap is zero if the agent is not at a left orange state.
When we are at the left orange state at the Type III model, there is no gap.
When we are at the left orange state at the Type II model, the gap is $\bracket{\frac{1}{2}+\frac{\rho}{H} } \cdot H - \frac{1}{H} \cdot H = \rho$.
When we are at the left orange state at the Type I model, the gap is $\bracket{\frac{1}{2}+\frac{\rho}{2H} } \cdot H - \frac{1}{H} \cdot H = \frac{\rho}{2}$.
By a rescaling of the number of states, length of the horizon, MDP gap, and the absolute constants, the promised lower bound is established.

\end{proof}

\section{GAP-DEPENDENT UNSUPERVISED EXPLORATION FOR MULTI-ARMED BANDIT AND MDP WITH A SAMPLING SIMULATOR}\label{appendix-section:bandit-and-mdp-with-generative-model}

\paragraph{Multi-Armed Bandit.}
The following result is from Theorem 33.1 in \citep{lattimore2020bandit}. For completeness, we restate the result and the proof here.
\begin{lemma}[Uniform Exploration]
    Consider a Bernoulli bandit with $A$ arms and a minimum non-zero expected reward gap $\rho > 0$. Consider the following policy: in the exploration phase an agent uniformly pulls each arm and collects rewards for $K = T/A$ rounds, and in the planning phase the agent chooses the arm with the highest empirical rewards.
    Then 
    \begin{enumerate}[leftmargin=*]
        \item the output is $(\epsilon, \delta)$-correct for $\epsilon < \rho$ if 
        \( T \eqsim \frac{A}{\rho^2} \log \frac{A}{\delta} \);
        \item the expected error is at most 
        \( \Ebb_\pi [V^* - V^\pi] \lesssim A \exp ( - \rho^2 T  / A ) \propto \exp(-T). \)
    \end{enumerate}
\end{lemma}
\begin{proof}
Let us denote the expected reward of an arm $a$ as $r_a$, and denote the empirical reward of an arm $a$ as $\widehat{R}_a =  (R_a^1 + \cdots + R_a^K) / K$.
Suppose $a$ is the best arm, and $a'$ is the arm with highest empirical reward, then
\begin{align*}
    \Pbb\{ a' \ne a \}
    &= \Pbb\set{ \widehat{R}_{a'} > \widehat{R}_a } \\
    &= \Pbb\set{ \bracket{\widehat{R}_{a'} - r_{a'}} - \bracket{ \widehat{R}_a - r_{a}} > r_a - r_{a'} } \\
    &\le \Pbb\set{ \bracket{\widehat{R}_{a'} - r_{a'}} - \bracket{ \widehat{R}_a - r_{a}} > \rho } \\
    &\lesssim A \exp ( -\rho^2 K  ) \eqsim A \exp ( - \rho^2 T  / A ).
\end{align*}
The proof is completed by noting that $0\le r_a \le 1$.
\end{proof}

\paragraph{MDP with a Sampling Simulator.}
Now we consider gap-TAE for MDP with a sampling simulator.
The algorithm is simple: in the exploration phase we sample $N$ data at each pair $(x,a)$ and compute an empirical transition probability $\widehat{\Pbb}$;
in the planning phase we compute the optimal value function over $\widehat{\Pbb}$, and output the induced greedy policy $\pi$. 
Mathematically speaking, the policy $\pi$ is given by
\begin{equation*}
    \begin{cases}
        \widehat{Q}^*_h (x,a) = r_h(x,a) + \widehat{\Pbb} \widehat{V}^*_{h+1} (x,a), \\
        \widehat{V}^*_h (x) = \max_a \widehat{Q}^*_h (x,a), \\
        \pi_h(x) = \arg\max_a \widehat{Q}^*_h (x,a).
    \end{cases}
\end{equation*}
Next we justify the sample complexity of this algorithm.

\begin{lemma}[Good events]\label{lemma:error-under-good-event-generative}
    Consider the following two events
    \[G := \set{\forall x, a, h,\ \abs{  (\widehat{\Pbb} - {\Pbb}) V^*_{h+1} (x,a) } < \frac{\rho}{2H}},\qquad 
    E:= \set{\forall x, h,\  V^*_h(x) - V^\pi_h(x) = 0 },\]
    then $G$ implies $E$.
\end{lemma}
\begin{proof}
    Assume $G$ holds, and define $E_h := \set{\forall x,\  V^*_h(x) - V^\pi_h(x) = 0 }$.
    We prove $E$ holds by induction over $E_h$ for $h\in \{ H+1, H, \dots, 1 \}$.
    First $E_{H+1}$ holds by definition. 
    Next suppose that $E_{h+1}, \dots, E_{H+1}$ holds, and consider $E_h$.
    Since $G$ holds, we have that for every $x$,
    \begin{equation*} 
        {V}^*_h (x) - \widehat{V}^{\pi^*}_h(x) 
        = \Ebb_{\pi, \widehat{\Pbb} } \sum_{t \ge h}  ( {\Pbb} - \widehat{\Pbb}) {V}^*_{t+1} (x_t, a_t) < (H+1 - h) \cdot \frac{\rho}{2H} < \frac{\rho}{2}.
    \end{equation*}
    Since $E_t$ holds for $t \ge h+1$, we have that for every $x$,
    \begin{align*}
        \widehat{V}^\pi_h (x) - V^\pi_h(x) 
        &= \Ebb_{\pi, \widehat{\Pbb} } \sum_{t \ge h} (\widehat{\Pbb} - {\Pbb}) {V}^\pi_{t+1} (x_t, a_t) \\
        &= \Ebb_{\pi, \widehat{\Pbb} } \sum_{t \ge h}  ( \widehat{\Pbb} - {\Pbb}) {V}^*_{t+1} (x_t, a_t)\qquad (\text{since $E_t$ holds for $t \ge h+1$}) \\
        &< (H+1 - h) \cdot \frac{\rho}{2H} < \frac{\rho}{2}. \qquad (\text{since $G$ holds})
    \end{align*}
    The above two inequalities imply that for every $x$,
    \begin{align*}
        V^*_h(x) - V^\pi_h(x) 
        &= V^*_h(x) - \widehat{V}^{\pi^*}_h(x) + \widehat{V}^{\pi^*}_h(x) -  \widehat{V}^{\pi}_h(x) + \widehat{V}^{\pi}_h(x) -  V^\pi_h(x) \\
        &\le { V^*_h(x) - \widehat{V}^{\pi^*}_h(x) } + {\widehat{V}^{\pi}_h(x) -  V^\pi_h(x)} < \rho,
    \end{align*}
    which further yields that for every $x$ and $a=\pi_h(x)$,
    \[
        V^*_h(x) - Q^*_h(x,a) \le  V^*_h(x) - V^\pi_h(x) < \rho,
    \]
    which forces \( a \in \pi^*_h(x)\) since otherwise $V^*_h(x) - Q^*_h(x,a) \ge \rho$.
    Therefore, we have that for every $x$,
    \[
        V^*_h(x) - V^\pi_h(x) = Q^*_h(x,a) - Q^\pi_h(x,a) = \Pbb \bracket{V^*_{h+1} -  V^\pi_{h+1}  } (x,a) = 0,
    \]
    where the last equality holds since $E_{h+1}$ holds, and by this we show that $E_h$ holds, which completes our induction.
\end{proof}

\begin{lemma}[Probability of the good event]\label{lemma:good-event-prob-generative}
    \(
    \Pbb\set{ G^c } < 2 HSA \cdot \exp \bracket{- \frac{\rho^2 N }{2 H^4}}. \)
\end{lemma}
\begin{proof}
    This is by Hoeffding's inequality and a union bound over $x,a,h$.
\end{proof}

\begin{theorem}[Restatement of Theorem \ref{thm:with-generative-model}]\label{thm:with-generative-model-restate}
Suppose there is a sampling simulator for the MDP considered in the gap-TAE problem.
Consider exploration with the uniformly sampling strategy, and planning with the dynamic programming method with the obtained empirical probability.
If $T$ samples are drawn, where 
\[
T \ge \frac{2H^4 SA}{\rho^2} \cdot \log \frac{2HSA}{\delta},
\]
then with probability at least $1-\delta$, the obtained policy is optimal ($\epsilon = 0$).
\end{theorem}
\begin{proof}
Note that $T \ge \frac{2H^4 SA}{\rho^2} \cdot \log \frac{2HSA}{\delta}$ implies that $N = T/(SA) \ge \frac{2H^4}{\rho^2} \cdot \log \frac{2HSA}{\delta}$, then by Lemma~\ref{lemma:good-event-prob-generative}, we have that 
\[ \Pbb\set{G} \ge 1- \delta, \]
then according to Lemma~\ref{lemma:error-under-good-event-generative}, the policy is optimal with probability at least $1-\delta$.
\end{proof}

\end{document}